%% file: main.tex
\documentclass[11pt,a4paper,final]{article}
\usepackage[utf8]{inputenc}
\usepackage{graphicx} 
\usepackage{natbib}
\usepackage{amsthm, amsmath, amssymb, amsfonts}
\usepackage{booktabs} 
\usepackage[ruled,noend]{algorithm2e}

\SetCommentSty{mycommfont}

\usepackage{algorithmic}
\usepackage[left=1in,right=1in,top=1in,bottom=1in]{geometry}
\usepackage[T1]{fontenc}    
\usepackage{hyperref}       
\usepackage{url}            
\usepackage{booktabs}       
\usepackage{nicefrac}       
\usepackage{microtype}      
\usepackage{mathtools}
\usepackage{enumitem}
\usepackage{authblk}
\usepackage{cleveref}
\usepackage{caption}
\usepackage{subcaption}
\usepackage{authblk}
\usepackage{xcolor}

\usepackage[suppress]{color-edits}
\addauthor{dn}{red}
\addauthor{sw}{blue}

\title{Reconciling Model Multiplicity for Downstream Decision Making}
\author[1]{Ally Yalei Du\footnote{Denote alphabetical order.}}
\author[2]{Daniel Ngo$^*$}
\author[1]{Zhiwei Steven Wu}

\affil[1]{Carnegie Mellon University, \texttt{\{aydu, zstevenwu\}@cmu.edu}}
\affil[2]{University of Minnesota, \texttt{ngo00054@umn.edu}}

\newcommand{\declarecolor}[2]{\definecolor{#1}{RGB}{#2}\expandafter\newcommand\csname #1\endcsname[1]{\textcolor{#1}{##1}}}

\declarecolor{White}{255, 255, 255}
\declarecolor{Black}{0, 0, 0}
\declarecolor{Maroon}{128, 0, 0}
\declarecolor{Coral}{255, 127, 80}
\declarecolor{Red}{182, 21, 21}
\declarecolor{LimeGreen}{50, 205, 50}
\declarecolor{DarkGreen}{0, 90, 0}
\declarecolor{Navy}{0, 0, 128}

\definecolor{plotblue}{HTML}{377eb8}
\definecolor{plotorange}{HTML}{ff7f00}
\definecolor{plotgreen}{HTML}{4daf4a}

\hypersetup{
    colorlinks=true,
    citecolor=DarkGreen,
    linkcolor=Navy,
    filecolor=magenta,      
    urlcolor=black,
    pdfpagemode=FullScreen,
}

\date{}

\begin{document}
\maketitle
\begin{abstract}
    We consider the problem of \emph{model multiplicity} in downstream decision-making, a setting where two predictive models of equivalent accuracy cannot agree on the best-response action for a downstream loss function. We show that even when the two predictive models approximately agree on their individual predictions almost everywhere, it is still possible for their induced best-response actions to differ on a substantial portion of the population. We address this issue by proposing a framework that \emph{calibrates} the predictive models with regard to both the downstream decision-making problem and the individual probability prediction. Specifically, leveraging tools from multi-calibration, we provide an algorithm that, at each time-step, first reconciles the differences in individual probability prediction, then calibrates the updated models such that they are indistinguishable from the true probability distribution to the decision-maker. We extend our results to the setting where one does not have direct access to the true probability distribution and instead relies on a set of i.i.d data to be the empirical distribution. Finally, we provide a set of experiments to empirically evaluate our methods: compared to existing work, our proposed algorithm creates a pair of predictive models with both improved downstream decision-making losses and agrees on their best-response actions almost everywhere.
\end{abstract}
\newpage
\tableofcontents
\newpage
\input{macros}
\input{intro}
\input{related-work}

\input{problem_formulation}
\input{counterexamples}
\input{reconciling/reconciling_decision_regression}
\input{reconciling/finite_sample_grid}

%
\input{experiments}
%
\input{conclusion}

\newpage 
\bibliographystyle{plainnat}
\bibliography{refs}

\newpage
\appendix
\input{appendix/appendix_counterexample}
\input{appendix/reconcile_proof}
\input{appendix/finite_sample_proof}

\end{document}

%% file: macros.tex
\newtheorem{theorem}{Theorem}[section]
\newtheorem{lemma}[theorem]{Lemma}
\newtheorem{fact}[theorem]{Fact}
\newtheorem{claim}[theorem]{Claim}
\newtheorem{remark}[theorem]{Remark}
\newtheorem{corollary}[theorem]{Corollary}
\newtheorem{proposition}[theorem]{Proposition}
\newtheorem{assumption}[theorem]{Assumption}
\newtheorem{example}[theorem]{Example}
\newtheorem{definition}[theorem]{Definition}
\newtheorem{condition}[theorem]{Condition}

\newtheoremstyle{named}{}{}{\itshape}{}{\bfseries}{.}{.5em}{\thmnote{#3 }#1}
\theoremstyle{named}
\newtheorem*{namedtheorem}{Assumption}

\makeatletter
\def\namedlabel#1#2{\begingroup
   \def\@currentlabel{#2}%
   \label{#1}\endgroup
}
\makeatother

\newcommand{\unif}{\mathsf{Unif}}
\newcommand{\normal}{\mathsf{Normal}}
\newcommand{\snr}{\mathsf{snr}}
\newcommand{\PCR}{\texttt{PCR}}
\newcommand{\cA}{\mathcal{A}}
\newcommand{\cE}{\mathcal{E}}
\newcommand{\cC}{\mathcal{C}}
\newcommand{\cN}{\mathcal{N}}
\newcommand{\cB}{\mathcal{B}}
\newcommand{\cS}{\mathcal{S}}
\newcommand{\cZ}{\mathcal{Z}}
\newcommand{\cV}{\mathcal{V}}
\newcommand{\cO}{\mathcal{O}}
\newcommand{\cM}{\mathcal{M}}
\newcommand{\cY}{\mathcal{Y}}
\newcommand{\cX}{\mathcal{X}}
\newcommand{\cG}{\mathcal{G}}
\newcommand{\cL}{\mathcal{L}}
\newcommand{\cT}{\mathcal{T}}
\newcommand{\cH}{\mathcal{H}}
\newcommand{\cI}{\mathcal{I}}
\newcommand{\cD}{\mathcal{D}}
\newcommand{\cP}{\mathcal{P}}
\newcommand{\cR}{\mathcal{R}}
\newcommand{\vY}{\mathbf{Y}}
\newcommand{\vX}{\mathbf{X}}
\newcommand{\vZ}{\mathbf{Z}}
\newcommand{\vP}{\mathbf{P}}
\newcommand{\vw}{\boldsymbol{\omega}}
\newcommand{\vc}{\mathbf{c}}
\newcommand{\vp}{\mathbf{p}}
\newcommand{\vr}{\mathbf{r}}
\newcommand{\vu}{\mathbf{u}}
\newcommand{\vv}{\mathbf{v}}
\newcommand{\vx}{\mathbf{x}}
\newcommand{\vy}{\mathbf{y}}

\newcommand{\vlambda}{\boldsymbol{\lambda}}
\newcommand{\valpha}{\boldsymbol{\alpha}}
\newcommand{\vtheta}{\boldsymbol{\theta}}
\newcommand{\vbeta}{\boldsymbol{\beta}}
\newcommand{\vtau}{\boldsymbol{\tau}}
\newcommand{\vs}{\boldsymbol{s}}

\newcommand{\eps}{\epsilon}

\newcommand{\range}[1]{[\![#1]\!]}
\newcommand{\E}{\mathbb E}
\newcommand{\bP}{\mathbb P}
\newcommand{\R}{\mathbb R}
\newcommand{\I}{\mathbb{I}}

\newcommand{\abs}[1]{\left\lvert#1\right\rvert}
\newcommand{\norm}[1]{\left\lVert#1\right\rVert}
\newcommand{\ie}{{i.e.,~\xspace}}
\newcommand{\eg}{{e.g.,~\xspace}}
\newcommand{\etc}{{etc.}}
\newcommand{\argmax}{\mathrm{argmax}}
\newcommand{\argmin}{\mathrm{argmin}}

\newcommand{\PB}{\mathsf{PB}}
\newcommand{\BR}{\mathsf{BR}}
\newcommand{\stat}{\mathrm{stat}}
\newcommand{\DC}{\mathsf{DC}}

%% file: intro.tex
\section{Introduction}
\label{sec:intro}
In many applications, individual probability prediction is at the heart of a decision-making process. 
For example, in the Job Training Partnership Act (JTPA) training program \citep{bloom1997benefits}, a decision-maker may want to predict whether an individual is employed or not before assigning them to training; or in medical trials, a doctor wants to predict the probability that the patient has contracted a disease before recommending them a treatment. Since the hospital does not know the true individual probability that a particular patient is ill, they can only evaluate the individual probability predictions through its average outcome over a sufficiently large sample set. For a predictive task, the standard convention is to choose the model that maximizes \emph{accuracy}. However, previous work has shown that it is common to have multiple predictive models with similar accuracy but substantially different properties \citep{, chen2018interpretable, rodolfa2020case, d2022underspecification}. This phenomenon is called \emph{predictive} (or model) multiplicity, a line of work studied by \citet{breiman2001statistical, marx2020predictive, black2022model}.

In a predictive multiplicity scenario, the decision-maker may have two or more predictive models that are nearly equivalent in terms of accuracy but disagree on their predictions on many individual samples. In our motivating example (Figure~\ref{fig:counterexample}), the hospital has access to two models $f_1$ and $f_2$ predicting the probability of disease which are equally accurate on average over the entire population, but their predictions on a subpopulation may vastly differ. This disagreement in outcome prediction may have a disparate impact on the subpopulation if the hospital has to choose one predictor over the other to make important downstream decisions. For example, they might select a treatment based on the predicted probability that a patient has contracted a disease. Formally, given a predictive model $f$ and a decision-making loss function $\ell(y,\cdot)$, the decision-maker wants to choose a best-response action, \ie the action $a$ that minimizes $\E_{y \sim f}[\ell(', a)]$. When two models $f_1$ and $f_2$ have nearly equivalent accuracy but lead to different best-response actions, the decision-maker would not be able to identify which best-response action to take for individual patients. While predictive multiplicity offers great flexibility for the decision-maker in the model selection process, it also places an additional burden on the decision-maker to correctly navigate such freedom and justify how they use a predictive model to make downstream decisions. 
 
\citet{roth2023reconciling} attempts to address the model multiplicity issue by resolving prediction disagreement between models. Adapting techniques from the literature of multi-calibration \citep{hebert2018multicalibration}, they provide a procedure called "Reconcile" that updates the predictive models to minimize their disagreements and improve the accuracy of each model. However, we show simple settings where the reconciled predictions from \cite{roth2023reconciling} can lead the downstream decision-makers to take actions with substantially higher losses.
 We visually demonstrate this scenario in \Cref{fig:counterexample}. For a more detailed discussion on the limitation of prior work, see \Cref{sec:limitation} and \Cref{appendix:limitation}. This motivates the study of how to reconcile predictive multiplicity with an explicit focus on its impact on downstream decisions.

\begin{figure}[ht]
    \centering
    \includegraphics[width=\textwidth]{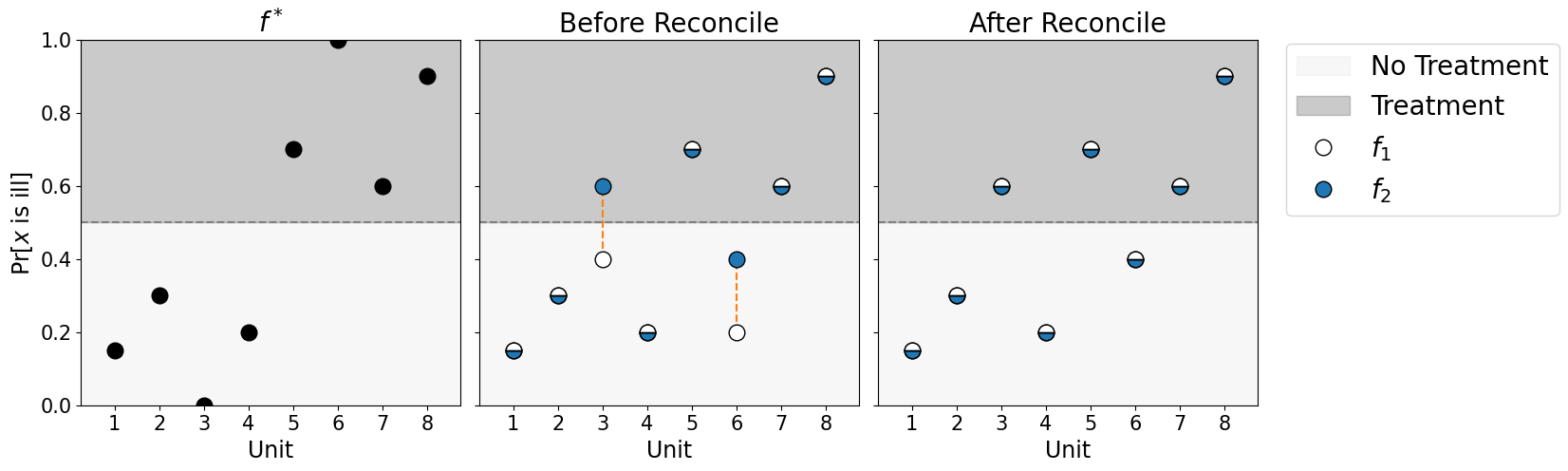}
    \caption{An illustrative example of the drawback in a prior work's attempt at addressing model multiplicity. Consider a stylized binary classification problem on a dataset with $8$ units (patients) and the hospital deciding between two actions (treatment vs. no treatment). Treatment is assigned if the predicted probability is above 1/2. \textbf{Left}: The true probability that each patient is labeled `ill'. \textbf{Middle}: The predicted probability that each patient is ill according to $f_1$ (white) and $f_2$ (blue). While these two predictors have almost the same accuracy, their individual probability predictions for patients $3$ and $6$ vastly differ. 
    \textbf{Right}: After running the Reconcile procedure of \citet{roth2023reconciling}, the individual probability predictions agree everywhere. 
    However, the best-response action of unit $3$ changed from correct (no treatment) to incorrect (treatment).
    If the hospital uses the updated $f_1$ to make their treatment recommendation, they would incur more loss than before had they not updated the predictor using Reconcile. This example is formalized in \Cref{thm:counterexample-reconcile}.}
    \label{fig:counterexample}
\end{figure}

In this work, our goal is to leverage tools from multi-calibration \citep{hebert2018multicalibration} to alleviate the model multiplicity issue in high-dimensional decision-making tasks for multiple decision-makers with multiple decision-making loss functions. Specifically, we show how the decision-maker can update a pair of predictors so they approximately agree on (1) individual predictions and (2) best-response actions for each individual in the downstream decision-making task. Our procedure ensures that the number of disagreements in best-response actions decreases over time, which enables the decision-makers to confidently use either of the updated predictors to justify their decisions.   

\paragraph{Overview of Paper.} We study the problem of reconciling model multiplicity for multiple downstream decision-making tasks, where the decision-makers have two predictive models with nearly equivalent accuracy but may lead to vastly different best-response actions for a significant number of individuals in the population. Our key contributions are summarized as follows.
\begin{itemize}
    \item In \Cref{sec:prelim}, we formulate the problem of model multiplicity from the perspective of the decision-makers. 
    In \Cref{sec:limitation}, we formalize our motivating example in Figure~\ref{fig:counterexample} and show that it is insufficient to only update two predictive models so that they have improved accuracy and nearly agree on their individual predictions almost everywhere.
    \item In \Cref{sec:reconcile}, we introduce an algorithm, ReDCal, that outputs predictive models that are (1) calibrated to a finite set of downstream decision-making tasks and (2) approximately agree on their predictions and best-response actions almost everywhere for each downstream task. 
    \item In \Cref{sec:finite-sample}, we extend our analysis to the setting where one does not have direct access to the true distribution and instead only has a validation dataset with samples drawn i.i.d from the underlying distribution. We show that the guarantees obtained using the empirical distribution can be translated to the unknown underlying distribution.
    \item Finally, in \Cref{sec:experiment}, we empirically evaluate the performance of the proposed algorithm on real-world datasets and show our improvement over the benchmark prior work in resolving disagreement in downstream decision-making tasks. 
\end{itemize}

%% file: related-work.tex
\subsection{Related Work}
\paragraph{Model multiplicity.} 
Within the literature on predictive multiplicity, our work builds off the line of work focusing on predicting individual probabilities \citep{marx2020predictive, d2022underspecification, black2022model, breiman2001statistical}, where solving an error minimization problem for some prediction tasks can lead to multiple solutions with roughly similar performance in terms of accuracy. \citet{sandroni2003reproducible} showed that one cannot empirically distinguish the outcomes from a predictor encoding the true individual probabilities from one without in isolation. \citet{nabil2008comparative, feinberg2008testing} provided comparative tests to differentiate between the true probability predictor and one that is not. Particularly, \citet{feinberg2008testing} relied on \emph{cross-calibration}, \ie calibration conditional on the predictions of both models to empirically falsify one of them. For downstream decision-making, \citet{garg2019tracking} worked on refining predictors and provided an algorithm that produces a predictor $f_3$ that is cross-calibrated with respect to both $f_1$ and $f_2$. An alternative framework studied by  \citet{globusharris2022algorithmic} seeks to update models that are sub-optimal for different subsets of the population, following the 'bug bounties' approach used by the software and security communities. 

Particularly relevant to our work is \citet{roth2023reconciling}, which aims to reconcile different predictors with equivalent errors such that the updated predictors both have lower errors compared to the initial models and approximately agree on their prediction on almost all units. Despite the similarity in our motivation, our results go beyond the binary classification setting considered in \citet{roth2023reconciling}. In our model, we consider reconciling predictors for both regression and multi-class classification problems, and their impact on the downstream decision-making tasks. In \Cref{sec:limitation}, we provide a numerical example where simply reconciling the probability predictions according to \citet{roth2023reconciling} can lead to additional losses in downstream decision-making tasks.

\paragraph{Calibration and Multi-calibration.} Our work draws on techniques from the growing literature on multi-calibration \citep{hebert2018multicalibration, kim2019multiaccuracy, dwork2019learning, shabat2020sample, jung2020moment, dwork2021outcome, jung2022batch, haghtalab2023unifying, deng2023happymap, noarov2023highdimensional}. Multi-calibration has been used as a notion of fairness as it guarantees calibration for any identifiable group. 

Within the framework of multi-calibration, the work most related to ours is that of \citet{zhao2021calibrating}, who considered decision calibration with respect to all classes of loss functions. Similar to us, \citet{zhao2021calibrating} takes the perspective of a decision-maker who wants to ensure the predictive models are \emph{indistinguishable} from the true probability when they are used to make downstream decisions. However, two decision-calibrated models can still disagree on their individual predictions and best-response actions for sufficiently many units. In our motivating example, a hospital with two decision-calibrated predictors may still want to make their predictive models approximately agree on their predictions for almost all individuals in the population and lead to the same downstream decision, \ie whether to recommend treatment or not based on the predicted probability that a patient has contracted a disease. We formalize this example in \Cref{sec:limitation} and provide empirical experiments to show our improvement over their result in \Cref{sec:experiment}.

An independent and concurrent work by \citet{globusharris2024model} considers ensembling multiple predictive models for high-dimensional downstream decision-making tasks. While their work also leverages techniques from multi-calibration, their goal is to output an ensembled predictor whose self-estimated expected payoff is accurate and whose induced policy has a payoff at least as high as the maximum self-assessed payoff of individual models. In contrast, our interest is in \textit{reducing} the downstream decision losses by resolving the differences between equivalent predictors and mitigating model multiplicity for decision-making.

%% file: problem_formulation.tex
\section{Problem Formulation}
\label{sec:prelim}
\paragraph{Notation.} Throughout this paper, we use subscripts $i$ to index different predictions, superscripts $t$ to index different time-steps, and $a$ to index actions. For $K \in \mathbb{N}$, we use the shorthand $[K] \coloneqq \{1,2, \cdots, K \}$. $\Delta(\cX)$ denote the set of possible distributions over $\cX$. 

We consider the prediction problem with random variables $x$ and $y$, where $x \in \cX$ represents the features and $y \in \cY$ represents the labels. We focus on the regression problem in which the label domain is real-valued and bounded: $\cY \subset [0,1]^d$. Our formulation also permits the multi-class classification problem by writing the label $y$'s as one-hot vectors, \eg for $\abs{\cY} = 3$, we can write $\cY = \{ [1, 0, 0], [0, 1, 0], [0, 0, 1]\}$.  

We denote $\cD \in \Delta(\cX \times \cY)$ as the true distribution over the pairs of features-label $(x,y)$. In practice, we will not have access to $\cD$, and instead only know a set of $n$ data points $D$ sampled i.i.d from $\cD$. In such case, we consider the dataset $D = \{(x_1, y_1), \cdots, (x_n, y_n) \}$ to be the \emph{empirical distribution} over $D$, which is a discrete distribution that place uniform weight $\nicefrac{1}{n}$ on each sample $(x,y) \in D$.

A predictor is a map $f: \cX \rightarrow [0,1]^d$. Our goal is to find the Bayes optimal predictor $f^*: \cX \rightarrow [0,1]^d$ such that for all $x \in \cX$, $f^*(x) = \E_{(x,y) \sim \cD}[y| x]$ is the \emph{conditional label expectation given $x$}. 

\subsection{Model Evaluation}

Given a predictor $f \in [0,1]^d$, we evaluate $f$ via its squared error, \ie its expected deviation from the true label. We formalize this objective in the following definition.
\begin{definition}[Brier Score] The squared error (also known as Brier score) of a predictor $f$ evaluated on distribution $\cD$ is given as:
    \begin{equation*}
        B(f, \cD) = \E_{(x,y) \sim \cD}[\norm{f(x) - y}_2^2]
    \end{equation*}
When we only have a dataset $D = \{ (x_1, y_1), \cdots, (x_n, y_n) \}$, the empirical Brier score is given as:
    \begin{equation*}
        B(f, D) = \frac{1}{n} \sum_{i=1}^n \norm{f(x_i) - y_i}_2^2
    \end{equation*}
\end{definition}

Note that we use the Brier score as our metric because it can be accurately estimated given access to only the samples from the distribution. Moreover, among all possible predictors, the Brier score is minimized by the Bayes optimal predictor $f^*$. 

\begin{lemma}
    Fix any distribution $\cD$ and let $f^*(x) = \E_{(x,y) \sim \cD}[y | x]$ represent the true conditional label encoded by $\cD$. Let $f: \cX \rightarrow [0,1]^d$ be any other model. Then we have $B(f^*, \cD) \leq B(f, \cD)$.
    %
\end{lemma}

Hence, given two predictors $f_1$ and $f_2$, if we can verify empirically from the observable data that $B(f_1, \cD) \leq B(f_2, \cD)$, then we can empirically falsify that $f_2$ encodes the true conditional label. 
\subsection{Downstream Decision-Making Tasks and Loss Functions}
Beyond our initial goal of finding a good estimate for the true conditional predictor $f^*$, we are also interested in using our predictors for downstream decision-making problems. Formally, we consider a loss minimization problem, where the decision-maker has a set of possible actions $\cA$ and a loss function $\ell: \cY \times \cA \rightarrow [0,d]$. Wlog, we only consider action set $\cA = [K]$, \ie there are $K$ possible actions. In this paper, we assume that the loss function does not directly depend on the features $x$ and is linear in $\cY$. That is, for each action $a\in \cA$, there exists some $\ell_a \in [0,1]^d$ such that
\begin{equation*}
    \ell(y, a) = \langle y, \ell_a \rangle
\end{equation*}
We write $\cL = \{ \ell: \cY \times \cA \rightarrow [0,d]\}$ to denote a finite family of loss functions. In general, we consider the setting with multiple different decision-makers, each using a different linear loss function in $\cL$. For any loss function $\ell: \cY \times \cA \rightarrow \mathbb{R}$, we can rescale each coordinate of $\ell_a$ to be between $[0,1]$. 
%

Given a predictor $f$ and a loss function $\ell$, the decision-maker selects an action $a \in \cA$ that minimizes the expected loss. We define the best-response policy taken by the decision-maker as follows.
\begin{definition}[Best-response policy]
    Given a loss function $\ell$, a predictor $f$ and the action set $\cA$, the best-response policy for $\ell$ is given as 
    \begin{equation*}
        \pi_\ell^\BR(f(x)) = \argmin_{a \in \cA} \langle f(x), \ell_a \rangle.
    \end{equation*}
\end{definition}
%
%
\subsection{Calibration}
In our setting, we consider the decision-maker only having access to some pre-trained predictors $f$ given by a third-party. For instance, a data scientist trained a pair of models on an image dataset without exact knowledge of how the downstream decision-maker will use such predictors. We may imagine the decision-maker as a hospital considering whether to recommend treatment to certain patients based on the predicted probability that the patient has contracted a skin disease. Since the hospital's treatment-recommendation algorithm is not known to the public (and the data scientist), we assume that the data scientist initially aim to minimize the squared error in their predictions. 

Since the hospital believes that the input predictors may not perform well according to their own loss function, they want the data scientist to convey trust through other performance guarantees of the predictors. One such guarantee is multi-calibration with respect to a finite set of loss functions $\cL \ni \ell$ and a set of events $\cE$ on the best-response policy, \ie if the loss function $\ell$ belongs to $\cL$, the decision-maker should be able to accurately compute the expected loss of choosing an action using the best-response policy $\pi_\ell^\BR$. Formally, we let $E_{a, \ell}(f(x), x)$ denote the action selection events: 
\begin{definition}[Best-response Events] Given a predictor $f$ and a loss function $\ell$, for each action $a \in [K]$, define the event 
\begin{equation*}
    E_{\ell, a}(f(x), x) = \mathbf{1} \{ x: \pi_\ell^\BR (f(x)) = a \}
\end{equation*}
and let $\cE = \{ E_{\ell, a} \}_{a \in [K], \ell \in \cL}$.
\label{def:best-response-event}
\end{definition}
%
%
Given a set of events $\cE$, we can define an approximate notion of multi-calibration with respect to $\cE$. 

\begin{definition}[$\beta$-approximate decision calibration] A predictor $f$ is $\beta-$decision calibrated with respect to the set of best-response events $\cE$ if for all $E_{\ell, a} \in \cE$, we have:
\begin{equation*}
    \norm{\E_{(x,y) \sim \cD}[(y - f(x)) \cdot E_{\ell, a}(f(x), x)]}_2 \leq \beta.   
\end{equation*}
\label{def:apprx-decision-calibration}
\end{definition}
This definition follows from an equivalent definition of decision calibration in \citet{zhao2021calibrating}. The main difference is we define calibration with respect to a set of events on the best-response actions following the formulation of multi-calibration for online learning in \citet{noarov2023highdimensional}  and a generalization of multi-calibration in \cite{deng2023happymap}.
This definition implies that if a predictor $f$ is $\beta-$decision calibrated with respect to the best-response events $\cE$, then the decision-maker can accurately estimate the expected loss from using $f$ to make decisions. 
\begin{lemma} [\citep{zhao2021calibrating}]\label{lem: cal_loss_est}
    For all $a, a' \in \cA, \ell \in \cL$, if $f$ is $\beta$-decision-calibrated with respect to the best-response events $\cE$ , then the loss estimation satisfies 
    \begin{align*}
        \left|\E_{(x,y) \sim \cD}[\ell(y, a') \cdot E_{\ell, a}(f(x), x)] - \E_{x \sim \cD_\cX}[\langle f(x), \ell_{a'}\rangle \cdot E_{\ell, a}(f(x), x)]\right| \leq \beta \sqrt{d}
    \end{align*}
\end{lemma}

%% file: counterexamples.tex
\subsection{Limitations of Prior Works}
\label{sec:limitation}
In this section, we show that improving the accuracy until the two predictors agree on their predictions almost everywhere is not a sufficient solution to our problem. In our analysis below, we consider a stylized problem with $\cY = \{0,1\}$ and $\cA = \{0,1\}$, \ie binary class and binary action space. A predictor here is $f: \cX \rightarrow [0,1]$, and the optimal predictor is $f^*(x) = \Pr_{(x', y') \sim \cD}[y' = 1 | x'=x]$. As shorthand, we denote $f_1(1)$ as the probability of unit $1$ being labeled $1$. The loss is defined as
\begin{align}
    &\ell(0,0) = \langle [1,0], [0,1] \rangle = 0, \quad 
    \ell(1,0) = \langle [0,1], [0,1] \rangle = 1, \nonumber\\
    &\ell(0,1) = \langle [1,0], [1,0] \rangle = 1, \quad 
    \ell(1,1) = \langle [0,1], [1,0] \rangle = 0,
    \label{eq:counterexample-loss}
\end{align}
That is, for any $x$, the best-response policy is to take action $0$ if $f(x) \leq \nicefrac{1}{2}$ and action $1$ otherwise. 

\paragraph{Reconcile individual predictions.}
Prior work by \citet{roth2023reconciling} considers the model multiplicity problem for individual probability predictions. Their proposed algorithm, Reconcile (\Cref{alg: reconcile-brier}), returns a pair of predictors that has a smaller Brier score than the input predictors and approximately agree on their predictions on almost all units. In the following theorem, we show that the best-response policy induced by the predictors updated by Reconcile might lead to a higher expected loss than the ones they started with.

\begin{theorem}
For any $\alpha \in (0,\nicefrac{1}{3})$, $\eta \in (0,1)$, there exists a pair of predictors $f_1, f_2$ with equivalent accuracy such that after running \Cref{alg: reconcile-brier}, the output models $f_1^T, f_2^T$ agree on their individual predictions everywhere, but there exists a loss function $\ell(y,a)$ such that $f_1^T$ and $f_2^T$ induce worse losses compared to the original models.  
\label{thm:counterexample-reconcile}
\end{theorem}
\begin{proof}

Consider a setting with $\cX = [2]$ and $\Pr[x = 1] = \Pr[x=2] = 0.5$. For any $0 < \alpha < \nicefrac{1}{3}$, let $\phi \geq \alpha$, we consider the two predictors $f_1, f_2$ defined as follows: 
\begin{align}
    f_1(1) = \frac{1}{2} - \frac{\phi}{2}, \quad  f_1(2) = \frac{1}{2} - \frac{3\phi}{2}, \quad 
    f_2(1) = \frac{1}{2} + \frac{\phi}{2}, \quad f_2(2) = \frac{1}{2} - \frac{\phi}{2}
\end{align}
and the true probability of each unit being labeled $1$ are $f^*(1) = 0$ and $f^*(2) = 1$.

The Brier scores of $f_1$ and $f_2$ differ only by $\phi^2$, but their individual predictions differ for both feature $x = 1$ and $x=2$. We can run \Cref{alg: reconcile-brier} and patch $f_1$ to get the updated model $f_1^T$ with 
\begin{equation*}
    f_1^T(1) = \frac{1}{2} + \frac{\phi}{2} = f_2(1), \quad f_1^T(2) = \frac{1}{2} - \frac{\phi}{2} = f_2(2).
\end{equation*}
Consider the loss function $\ell$ defined as  \Cref{eq:counterexample-loss}.
The change in expected loss after patching $f_1$ is
\begin{align*}
    &\E_{(x,y) \sim \cD} [\ell(y, \pi_\ell^\BR(f_1^T(x))) - \ell(y, \pi_\ell^\BR(f_1(x)))] 
    = \frac{1}{2}> 0.
\end{align*}
Therefore, no matter how small we let $\alpha$ and $\eta$ be, the loss of predictor $f_1$ increases by a constant amount after running \Cref{alg: reconcile-brier}. For the theoretical guarantees of \Cref{alg: reconcile-brier}, see \Cref{appendix:limitation}.   
\end{proof}
Moreover, we provide a counterexample to show that it is insufficient to only ensure each individual predictor is approximately decision-calibrated using \Cref{alg: decision_cali}. 
\paragraph{Decision-Calibrated predictions.} 
Another baseline algorithm we consider is to run Decision Calibration (Algorithm~\ref{alg: decision_cali}) separately for both $f_1, f_2$. However, in the following theorem, we show that the updated predictors $f_1'$ and $f_2'$ can still disagree with each other on the best-response actions for substantially many units, indicating room for further improvement.

\begin{theorem}
    For any $\eta \in (0, \nicefrac{1}{4})$ and $\beta \in (0, \nicefrac{1}{2})$, there exists a pair of predictors $f_1, f_2$ and a loss function $\ell(y, a)$ such that after running Decision-Calibration \citep{zhao2021calibrating}, the resulting models $f_1^T, f_2^T$ are $\beta$-decision-calibrated with respect to the loss function $\ell$. There exists a set of units $x$ with probability mass $2\eta$ where $f_1^T$ and $f_2^T$ disagree on the individual best-response actions. 
\end{theorem}
\begin{proof}
For any $\eta \in (0, \nicefrac{1}{4}), \beta \in (0, \nicefrac{1}{2})$, let $\cX = [4]$, with $\Pr[1] = \Pr[4] = \nicefrac{1}{2}-\eta $ and $\Pr[x=2] = \Pr[x=3] = \eta$. Consider the predictors $f_1, f_2$ as follows:
\begin{align}
    &f_1(1) = f_1(2) = f_2(1) = f_2(3) =  \frac{\beta}{4} - 2\eta \beta + 2\eta, \\
    &f_1(3) = f_1(4) = f_2(2) = f_2(4)= 1-\frac{\beta}{2}, \\
    &f^*(1) = \frac{\beta}{2}, ~~ f^*(2) = f^*(3) = f^*(4) = 1-\frac{\beta}{2}.
\end{align}
Notice that
\begin{align*}
    \frac{\alpha}{4} - 2\eta \alpha + 2\eta = \frac{\alpha}{4} + 2\eta(1- \alpha) < \frac{1}{4} + \frac{1}{4} =\frac{1}{2}.
\end{align*}
The best-response policy for each predictor is
\begin{align}
    \pi_\ell^\BR(f_1(1)) = \pi_\ell^\BR(f_1(2)) = 0,~~& \pi_\ell^\BR(f_1(3)) = \pi_\ell^\BR(f_1(4)) = 1, \\
    \pi_\ell^\BR(f_2(1)) = \pi_\ell^\BR(f_2(3)) = 0,~~& \pi_\ell^\BR(f_2(2)) = \pi_\ell^\BR(f_1(4)) = 1.
\end{align}
For each best-response event, we have 
\begin{align*}
    \E_{(x,y) \sim \cD}[f^*(x) E_0(f_1(x), x)] 
    &= \E_{(x,y) \sim \cD}[f_1(x) E_0(f_1(x), x)], \\
    \E_{(x,y) \sim \cD}[f^*(x) E_1(f_1(x), x)] 
    &= \E_{(x,y) \sim \cD}[f_1(x) E_1(f_1(x), x)],\\
    \E_{(x,y) \sim \cD}[f^*(x) E_0(f_2(x), x)] 
    &= \E_{(x,y) \sim \cD}[f_2(x) E_0(f_2(x), x)], \\
    \E_{(x,y) \sim \cD}[f^*(x) E_1(f_2(x), x)] 
    &= \E_{(x,y) \sim \cD}[f_2(x) E_1(f_2(x), x)].
\end{align*}
That is, $f_1, f_2$ are already decision-calibrated, so running Decision Calibration (\Cref{alg: decision_cali}) will not further improve either of the two predictors. However, based on our definition of disagreement events (\Cref{def:disagree_event}), we still have
\begin{align*}
    E_{0,1} = \{2\},~~ E_{1,0} = \{3\},
\end{align*}
each with size
\begin{align*}
    \mu(E_{0,1}) = \mu(E_{1,0}) = \eta.
\end{align*}
We observe that $f_1$ and $f_2$ still disagree on the best-response action for units $x=2$ and $x=3$. 
We can further reduce the differences in best-response actions using our algorithm \Cref{alg: reconcile}. 
\end{proof}

%% file: reconciling/reconciling_decision_regression.tex
\section{Reconcile for Decision Making}
\label{sec:reconcile}

Suppose we are given two predictors $f_1, f_2: \cX \rightarrow [0,1]^d$. We consider the model multiplicity problem with respect to the downstream decision-making problem -- where $f_1, f_2$ with nearly equivalent accuracy differ in their induced decision-making policies, but we cannot falsify either of the two from the data. Informally, our goal is to return a pair of models $f_1', f_2'$ such that: (1) for both $i \in \{1,2\}$, $f_i'$ is more accurate than $f_i$ in terms of Brier score; (2) for both $i \in \{1,2\}$, the best-response policy induced by $f_i'$ has no larger expected loss than that of $f_i$; (3) $f_1'$ and $f_2'$ approximately agree almost everywhere, indicating limited room for additional improvement.

To this end, we are interested in the region where the two predictors disagree substantially with respect to the downstream decision-making task. We define the disagreement region as follows:

\begin{definition}[Disagreement Event] \label{def:disagree_event}
    For $f_1, f_2$, margin $\alpha > 0$, and a loss function $\ell$, the disagreement event is defined for a pair of best-response  actions $a_1, a_2 \in \cA $ where $a_1 \neq a_2$ as
    \begin{align*}
    E^\alpha_{\ell, a_1,a_2}(f_1(x), f_2(x), x) = \I \Big[x \in \big \{x:& \pi_\ell^\BR(f_1(x)) = a_1 , \pi_\ell^\BR(f_2(x)) = a_2, \\ 
    &\langle f_1(x), \ell_{a_2} - \ell_{a_1} \rangle > \alpha ~\text{or}~ \langle f_2(x), \ell_{a_1} - \ell_{a_2} \rangle > \alpha \big\} \Big],
    \end{align*}
\end{definition}

As shorthand, we denote $E_{\ell, a_1, a_2}(x) = E^\alpha_{\ell, a_1, a_2}(f_1(x), f_2(x), x)$ when the predictors $f_1, f_2$ and the margin $\alpha$ are clear from context. For a finite family of loss functions $\cL$, we can always iterate through $\cL$ to identify the tuple $(\ell, a_1, a_2)$ that defines a disagreement region between $f_1$ and $f_2$. 

We say the two models approximately agree with each other when the size of the disagreement event is small enough, \ie its probability mass $\mu(E_{\ell, a_1, a_2})$ on the underlying distribution $\cD$ is small.
%
\subsection{The Reconcile Procedure} \label{sec: reconcile_distribution}
In this section, we propose our main algorithm, ReDCal (\Cref{alg: reconcile}). Whenever the decision-maker observes a large disagreement event $E_{\ell, a_1, a_2}$, the best-response action and its corresponding expected loss given by at least one of the predictors must be incorrect. 
For example, at time step $t$ and a unit $x$, if the gap between the losses after taking $a_1$ and $a_2$ according to $f_2$ is substantially different from the loss gap observed on the data, then the decision-maker can induce that $f_2$ must have been wrong in its prediction for $x$. 
Then, the decision-maker would want to 'patch' predictor $f_2$ in this time-step. 

The calibration procedure within each time-step is divided into two stages. In the first stage, we update model $f_2$ to $f_2'$ by minimizing the mean prediction error on the disagreement event, \ie minimizing $\E[\|f_2'(x) - y | E_{\ell, a_1, a_2}(x) = 1\|]$. Following the intuition from multi-calibration, updating predictor $f_2$ in this manner would improve the Brier score and produce a more accurate predictor.
However, the updated model $f'_2$ is not guaranteed to induce the correct best-response action and could instead induce some other actions that might lead to a larger expected loss. To cope with this, in the second stage, we further update $f_2'$ to a model $f_2''$ that is approximately decision-calibrated within event $E_{\ell, a_1, a_2}$ using \Cref{alg: decision_cali}. Since the loss estimation given by $f_2''$ is accurate for all best-response events within $E_{\ell, a_1,a_2}$ and we are taking actions to minimize estimated loss, we can now safely take the best-response action induced by $f_2''$. The formal description of the algorithm is given by Algorithm~\ref{alg: reconcile}. 

\begin{algorithm}[ht]
\caption{Decision Calibration}
    \begin{algorithmic}[1]
    \label{alg: decision_cali}
    \renewcommand{\algorithmicrequire}{\textbf{Input:}}
    \renewcommand{\algorithmicensure}{\textbf{Output:}}
    
    \REQUIRE Predictor $f$, loss family $\cL$, $\beta > 0$, event $E$
    \STATE Let $f^0 = f$.
    \WHILE{$f^t$ is not $\beta$-multicalibrated with respect to events $E_{\ell,a} \cap E$ for some $\ell \in \cL$}
        \STATE Let $\ell^t, a^t = \argmax_{\ell, a} \|\E_{(x,y) \sim \cD}[(y-f^t(x))E_{\ell, a}(f^t(x), x)]\|_2$
        \STATE Let $\phi^t = \E_{(x,y) \sim \cD}[y - f^t(x) | E_{\ell^t, a^t}(f^t(x), x) = 1]$
        \STATE Patch $f^{t+1}(x) = \mathrm{proj}_{[0,1]^d}(f^t(x) + \phi^t E_{\ell^t, a^t}(f^t(x), x))$
        \STATE $t = t+1$.
    \ENDWHILE
    \ENSURE $f^t$
    \end{algorithmic}
\end{algorithm}

\begin{algorithm}[ht]
\caption{Reconcile Decision Calibration (ReDCal)}
    \begin{algorithmic}[1]
    \label{alg: reconcile}
    \renewcommand{\algorithmicrequire}{\textbf{Input:}}
    \renewcommand{\algorithmicensure}{\textbf{Output:}}
    
    \REQUIRE $f_1,f_2,\cL, \eta > 0, \alpha > 0, \beta > 0$
    \STATE Let $f_1^{0} = f_1, f_2^{0} = f_2$ and $t = 0$.  
    \WHILE{$\mu(E_{\ell, a_1, a_2}) \geq \eta$ for some $a_1, a_2 \in \cA$ and $\ell \in \cL$}
        \STATE Let $\ell^t, a_1^t, a_2^t = \argmax_{\ell, a, a'} \mu(E_{\ell, a,a'})$, $E^t = E_{\ell^t, a_1^t, a_2^t}$.
        \STATE Pick 
        \begin{align*}
            i^t = \argmax_{i \in \{1,2\}} \big|&\E_{(x,y) \sim \cD}
            [\ell^t(y, a_1^t) - \ell^t(y, a_2^t) | E^t(x) = 1] \\
            &-\E_{x \sim \cX} [\ell^t(f_i(x), a_1^t) - \ell^t(f_i(x), a_2^t) | E^t(x) = 1]\big|.
        \end{align*}
        \STATE Denote $f_{i^t}^{t}$ as $f_i^t$.  Let $\phi^t = \E_{(x,y) \sim \cD}[y | E^t(x) = 1] - \E_{x \sim \cD_\cX}[f_{i}^t(x) | E^t(x) = 1]$.
        \STATE Patch $f^{t}(x) = \mathrm{proj}_{[0,1]^d}(f_{i}^t(x) + \phi^t E^t(x))$.
        \STATE Let $f_{i}^{t+1} = $ Decision-Calibration($f^t$, $\cL$, $\beta$, $E^t$). $t = t+1$. 
    \ENDWHILE
    \ENSURE $f_1^{t}, f_2^{t}$
    \end{algorithmic}
\end{algorithm}

We provide the theoretical guarantees of our proposed algorithm below. At a high level, \Cref{alg: reconcile} produces a pair of models with improved accuracy and approximately agrees on the best-response action almost everywhere. For the formal proofs of this section, see \Cref{appendix:reconcile_proofs}. 

\begin{theorem}
    For any pair of models $f_1, f_2: \cX \rightarrow [0,1]^d$, any distribution $\cD$, family of loss functions $\cL$, any loss margin $\alpha > 0$, disagreement region mass $\eta > 0$, and decision-calibration tolerance $\beta > 0$, Algorithm~\ref{alg: reconcile} updates $f_1$ and $f_2$ for $T_1$ and $T_2$ time-steps, respectively,
    and outputs a pair of models $(f_1^{T}, f_2^{T})$, such that:
    \begin{enumerate}
        \item \Cref{alg: reconcile} terminates within $T = T_1 + T_2 \leq \frac{4 \cdot d \cdot \left( B(f_1, \cD) + B(f_2, \cD) \right)}{\alpha^2 \eta}$ time-steps.
        \item The Brier scores of the final models are lower than that of the input models $(f_1, f_2)$:
        \begin{equation*}
            B(f_1^{T}, \cD) \leq B(f_1, \cD) - T_1 \cdot \frac{\alpha^2 \eta}{4d} \quad \text{and } \quad  B(f_2^{T}, \cD)\leq B(f_2, \cD) - T_2 \cdot \frac{\alpha^2 \eta}{4d} 
        \end{equation*}
        \item All the downstream decision-making losses of the final models do not increase by much compared to that of the input models $(f_1, f_2)$: for each $i \in \{1,2\}$ and for all $\ell \in \cL$,
        \begin{equation*}
            \E_{(x,y) \sim \cD}[\ell(y, \pi_\ell^\BR(f_i^T(x))] - \E_{(x,y) \sim \cD}[\ell(y, \pi_\ell^\BR(f_i(x))] \leq T_i \beta \sqrt{d} K 
        \end{equation*}
        \item The final models approximately agree on their best-response actions almost everywhere. That is, the disagreement region $E_{\ell, a_1, a_2}$ calculated using $f_1^T, f_2^T$ has small mass. For all $\ell \in \cL$,
        \begin{equation*}
            \mu(E_{\ell, a_1, a_2}) < \eta \quad \text{ for all }  a_1, a_2 \in \cA \quad \text{s.t } a_1 \neq a_2
        \end{equation*}

    \end{enumerate}
\label{thm:reconcile}
\end{theorem}

\begin{remark}
    Note that in the third result of \Cref{thm:reconcile}, the increase in downstream decision-making loss at each time-step only depends on the decision-calibrate tolerance $\beta$, dimension $d$, and number of actions $K$. Since the total number of time-steps in \Cref{thm:reconcile} does not depend on $\beta$, we can set $\beta = \frac{\alpha}{T \sqrt{d} K}$ to ensure the loss of taking the best-response action does not degrade by more than $\alpha$. Moreover, in our empirical experiments (\Cref{sec:experiment}), we observe that the loss of taking the best-response action only increases minimally. 
\end{remark}

%% file: reconciling/finite_sample_grid.tex
\subsection{Finite Sample Analysis}
\label{sec:finite-sample}

In \Cref{sec: reconcile_distribution}, we have presented an algorithm, ReDCal, to reconcile two predictors assuming the decision-makers have direct access to the probability distribution $\cD$. However, in practice, the decision-makers will only have access to a dataset $D = \{(x_1, y_1), \cdots, (x_n, y_n)\}$ containing $n$ i.i.d samples drawn from $\cD$. In this section, we will instead run \Cref{alg: reconcile} on the empirical distribution over $D$ and show that its guarantees can translate to the underlying distribution $\cD$ with high probability. To prevent data leakage, it is important to assume that the dataset $D$ is drawn independently of the predictors $f_1$ and $f_2$, \ie the dataset contains freshly drawn data that was not used to train either of the predictors that we want to reconcile. For the formal proofs, see \Cref{appendix:finite_sample_proofs}. 

At a high level, since the samples in $D$ are independently and identically distributed, we can apply Chernoff-Hoeffding inequality to show that, with high probability, the in-sample quantities are approximately equal to out-sample quantities. We summarize the results in the theorem below.

\begin{theorem} \label{thm: sample_reconcile}
    Fix any distribution $\cD$ and dataset $D \sim \cD$ containing $n$ samples drawn i.i.d from $\cD$. For any pair of predictors $f_1, f_2: \cX \rightarrow [0,1]^d$, family of loss functions $\cL$, loss margin $\alpha > 0$, disagreement region mass $\eta > 0$, and decision-calibration tolerance $\beta > 0$, \Cref{alg: reconcile} run over the empirical distribution $D$ updates $f_1$ and $f_2$ for $T_1$ and $T_2$ time-steps, respectively, and outputs a pair of predictors $(f_1^{T}, f_2^{T})$ such that, with probability at least $1-\delta$ over the randomness of $D \sim \cD^n$,  
    \begin{enumerate}
        \item The total number of time-steps for \Cref{alg: reconcile} and \Cref{alg: decision_cali} is 
        \begin{equation*}
            T = T_1 + T_2 \leq \frac{2d}{\min\{\beta^2, \eta \alpha^2/4d\}}
        \end{equation*}
        \item For $i \in \{1,2\}$, the Brier scores of the final models are lower than that of the input models:
        \begin{align*}
             B(f_i^{T_i}, \cD) \leq B(f_i, \cD) - (T_i / 2) \cdot \min\left\{\beta^2, \eta \alpha^2 / (4d)\right\}
        \end{align*}
        \item For $i \in \{1,2\}$ and for all $\ell \in \cL$, the downstream decision-making losses of the final models do not increase by much compared to that of the input models: 
        \begin{align*}
            &\E_{(x,y) \sim \cD}[\ell(y, \pi_\ell^\BR(f_i^T(x)) - \ell(y, \pi_\ell^\BR(f_i(x))] 
            \leq 2T_i \beta \sqrt{d} K
        \end{align*}
        \item The final models approximately agree on their best-response actions almost everywhere. That is, the disagreement region $E_{\ell, a_1, a_2}$ calculated using $f_1^T, f_2^T$ has small mass: $\forall \ell \in \cL$,
        \begin{equation*}
            \mu(E_{\ell, a_1, a_2}) \leq 2\eta \quad \text{ for all }  a_1, a_2 \in \cA \quad \text{s.t } a_1 \neq a_2
        \end{equation*}
    \end{enumerate}
    if $n \geq \Omega\left(d/ (\eta^2 \min\{\beta, \eta \alpha^2/d\}) \cdot \left( \ln (K) + \ln(\abs{\cL}) + d \ln \left(d/\min\{\beta, \eta \alpha^2/d\}\right) + \ln (1/\delta)\right)\right). $
\end{theorem}


%% file: experiments.tex
\section{Experiments}
\label{sec:experiment}
In this section, we complement our theoretical results with a set of experiments on real-world datasets to show our improvement in decreasing decision-making loss compared to prior work in calibration. 

\subsection{Imagenet Multi-class Classification}

\paragraph{Experiment Setup.} 
We use the ImageNet dataset \citep{ImageNet} and two pre-trained models provided by pyTorch (inception-v3 \citep{szegedy2015rethinking} and resnet50 \citep{he2015deep}). Among the $50000$ validation samples, we use $40000$ samples for calibration and $10000$ samples for testing. 

We investigate how the downstream decision loss changes with the four calibration algorithms: Reconcile (\Cref{alg: reconcile-brier}), Decision-Calibration (\Cref{alg: decision_cali}), ReDCal (\Cref{alg: reconcile}), and the combination of running ReDCal after Decision Calibration as post-process.
We run each calibration algorithm $500$ times. For each run, we 
first randomly draw $100$ classes from the $1000$ classes of ImageNet. Then, we randomly generate a loss function such that, for each $y \in \cY$, $a \in \cA$,  $\ell(y,a) \sim \normal(0,1)$. For each randomly generated loss function $\ell$, we compare the expected losses derived from the best-response policies based on predictors $f_1$ and $f_2$ against those based on the optimal predictor $f^*$. Formally, the loss gap at timestep $t$ using predictor $f_i$ is defined as 
\[\text{LossGap}(f_i^t) = \E_{(x,y) \sim \cD}[\ell(y, \pi_\ell^\BR(f_i^t(x))) - \ell(y, \pi_\ell^\BR(f^*(x)))].\]

The hyperparameters are chosen as follows: loss margin $\alpha = 0.001$,  disagreement region mass $\eta = 0.01$, decision-calibration tolerance $\beta = 0.00001$, and the number of actions $K=10$.

\paragraph{Results.} We compare the performance of ReDCal with the two baseline algorithms in terms of Brier scores and decision loss reduction. Furthermore, in \Cref{fig: imagenet_compare_classes}, we provide a comparison of the calibration algorithms' performance when the number of dimensions increases.

\paragraph{Brier score.} In \Cref{fig: imagenet_brier-score}, we compare the Brier score of ReDCal (\Cref{alg: reconcile}) with the two baseline algorithms on both the calibration and the test datasets. Compared to Reconcile, our algorithm decreases the Brier score by a smaller amount on the test dataset. The combined algorithm of Decision-Calibration with ReDCal as post-process achieves the most substantial decrease in the Brier score.

\begin{figure}[ht]
    \includegraphics[width = \textwidth]{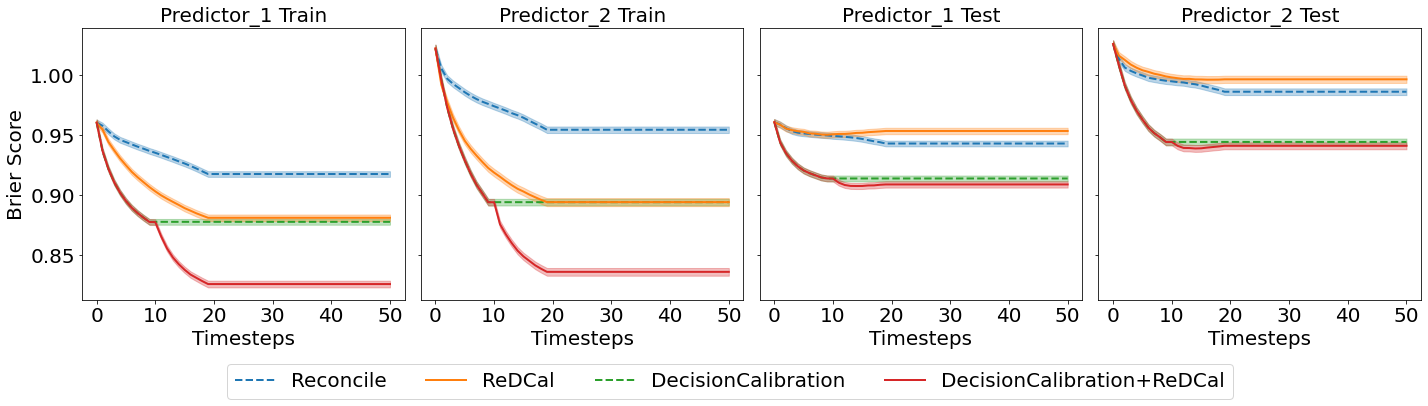}
    \caption{ReDCal decreases Brier score on Imagenet. Compared to Reconcile, our algorithm decreases the Brier score by a smaller amount on the test dataset. Decision-Calibration with ReDCal as post-process achieves the most substantial decrease in the Brier score. }
    \label{fig: imagenet_brier-score}
\end{figure}

\paragraph{Decision loss on calibration dataset.}
In \Cref{fig:imagenet_experiments}, we compare the decision gap of our proposed algorithm with the two baseline algorithms on the training dataset.
Compared to Reconcile, ReDCal converges within a similar number of time-step and decreases the loss by a larger amount on the test dataset. Moreover, ReDCal further decrease the loss when used as a post-process after Decision-Calibration terminates. 
\begin{figure}[ht]
     \centering     
     \begin{subfigure}[b]{\textwidth}
         \centering
         \includegraphics[width = \textwidth]{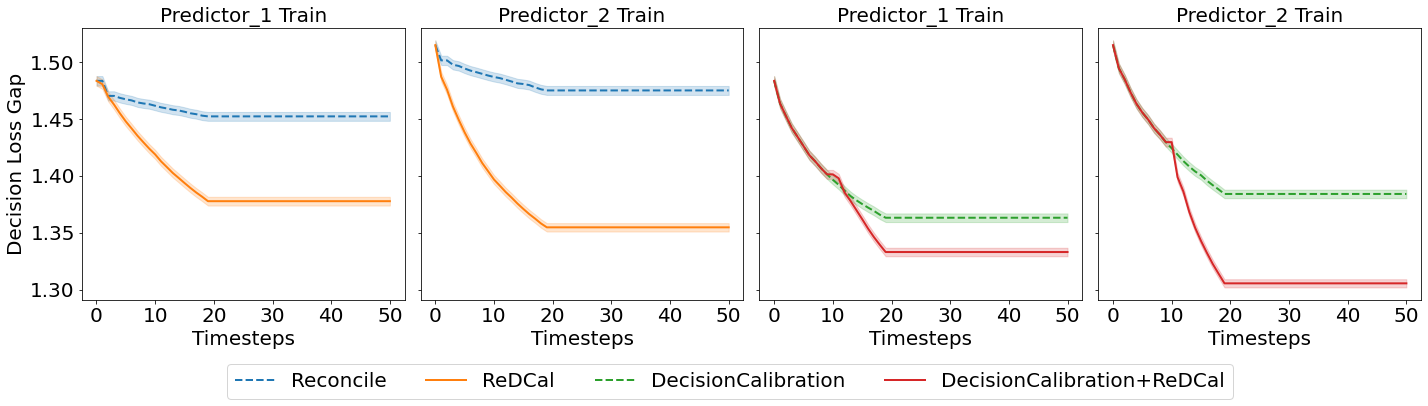}
         \caption{ReDCal decreases the decision loss on the validation partition of the ImageNet dataset. }
         \label{fig:imagenet_train_loss}
     \end{subfigure}
     \vfill
     \begin{subfigure}[b]{\textwidth}
         \centering
         \includegraphics[width = \textwidth]{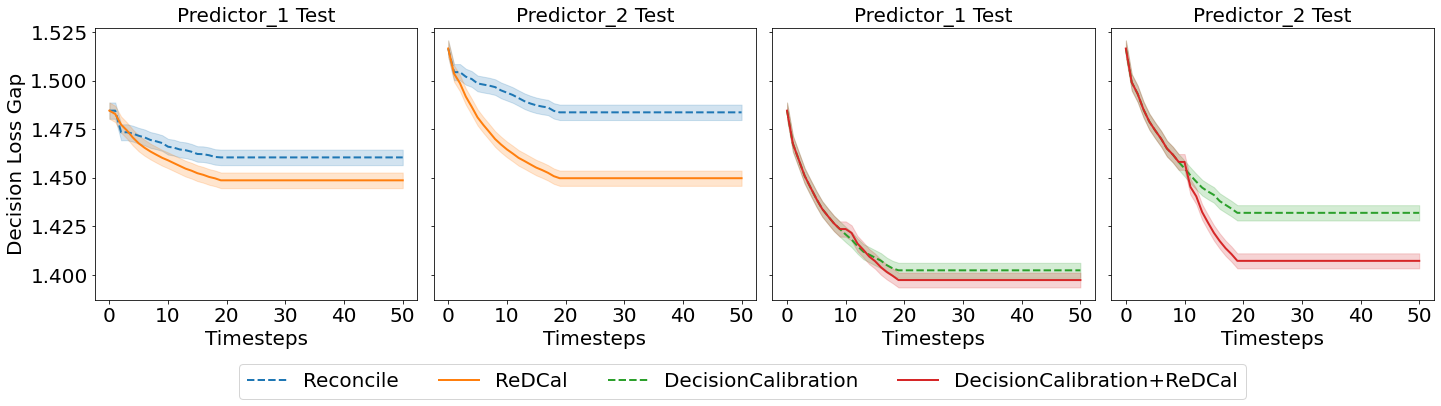}
         \caption{ReDCal decreases the decision loss on the test partition of the ImageNet dataset. }
         \label{fig:imagenet_test_loss}
     \end{subfigure}
        \caption{In \Cref{fig:imagenet_train_loss} and \Cref{fig:imagenet_test_loss}, we plot the gap between optimal loss had we know the true label $y$ and the loss from taking best-response actions induced by the calibrated predictors on the validation set and test set, respectively. In the left two figures, we compare \Cref{alg: decision_cali} (orange) with \Cref{alg: reconcile-brier} (blue). While the average loss of predictors updated using \Cref{alg: reconcile-brier} may increase on the test set, our algorithm quickly converges and produces predictors with lower decison-making loss. In the right two figures, we compare \Cref{alg: decision_cali} (green) to \Cref{alg: decision_cali} with an additional run of \Cref{alg: reconcile} (red) as post-process. We observe that running our algorithm as post-process can still further decrease the loss compared to just running \Cref{alg: decision_cali} on its own. Results are averaged over $10$ runs and the shaded region indicates $\pm 1$ standard errors.}
        \label{fig:imagenet_experiments}
\end{figure}
\paragraph{Decision loss comparison for high-dimensional classification problem.} In \Cref{fig: imagenet_compare_classes}, we compare the decision loss gap of our proposed algorithm with the two baseline algorithms on the testing dataset, using $d = 10, 100$, and $1000$ classes. We plot the average loss gap of the two predictors. 
The hyperparameters are: disagreement margin $\alpha = 0.1/d$, decision-calibration tolerance $\beta = 0.001/d$, disagreement region mass $\eta = 0.01$, number of actions $K = 10$.

\begin{figure}[ht]
    \centering
    \includegraphics[scale = 0.28]{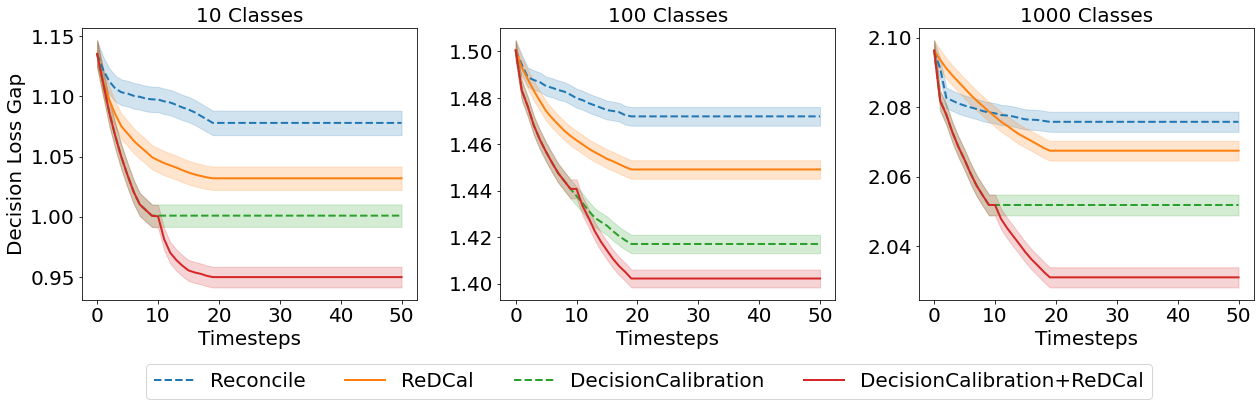}
    \caption{ReDCal decreases decision loss on Imagenet. The takeaway results are similar to \Cref{fig:imagenet_test_loss}. As the number of classes in the multi-class classification problem grows from $10$ to $1000$, ReDCal still outperforms Reconcile in decreasing decision loss on the test dataset. When we have $1000$ classes, ReDCal converges slower than Reconcile. Furthermore, ReDCal can further decrease the decision loss when it is used as a post-process after Decision Calibration terminates.}
    \label{fig: imagenet_compare_classes}
\end{figure}

\subsection{HAM10000 Multi-class Classification}
\paragraph{Experiment Setup.} 
We use the HAM10000 dataset \citep{tschandl2018ham} (licensed CC BY-NC $4.0$) on pigmented skin lesions to predict the probability that a patient has contracted one of $7$ possible skin diseases: 'akiec', 'bcc', 'bkl', 'df', 'nv', 'vasc', and 'mel'. We split the dataset into train/validation/test sets, with $20 \%$ of the data are used for validation and $20 \%$ are used for testing. We use the train set to train two neural networks using pyTorch with resnet50 \citep{he2015deep} and densenet121 \citep{huang2018densely} architectures and learn two models with around $88\%$ top-$1$ accuracy. From each model, we output the individual probability prediction for each of the $7$ possible labels. We use the validation set to calibrate the predictors using our proposed algorithm and the two baseline algorithms, and the test set to measure the final performance.   

We run each calibration algorithm $10$ times. At each run, we draw a fresh loss function created based on the loss function motivated by medical domain knowledge in \citet{zhao2021calibrating} and additional random noise drawn from $\normal(0,1)$. There are two possible actions for the decision-maker: treatment ($a = [1,0]$) or no treatment ($a = [0,1]$). Given a loss function $\ell$ and a predictor $f$, the decision-maker will choose an action that minimizes their loss. 

For each calibration algorithm, we calculate (1) the Brier score of the updated predictors and (2) the differences between the optimal loss had we known $y$ and the actual loss from taking the best-response actions induced by each predictor.  

The hyperparameters for \Cref{alg: reconcile} are chosen as follows: loss margin $\alpha = 0.1$, target disagreement region mass $\eta = 0.01$, and decision-calibration tolerance $\beta = 0.000001$. 

\paragraph{Results.} Similar to the experiment on the ImageNet dataset, we observe lower decision loss after running ReDCal compared to the baseline algorithm Reconcile. Moreover, we compare the performance of running Decision-Calibration on its own and using ReDCal as a post-process. 

\paragraph{Brier score.} In \Cref{fig:HAM-brier-score}, we compare the Brier score of ReDCal (\Cref{alg: reconcile}) with the two baseline algorithms on both the calibration and the test datasets. Compared to Reconcile, our proposed algorithm decreases the Brier score by a smaller amount.
\begin{figure}[ht]
     \centering
         \includegraphics[width=\textwidth]{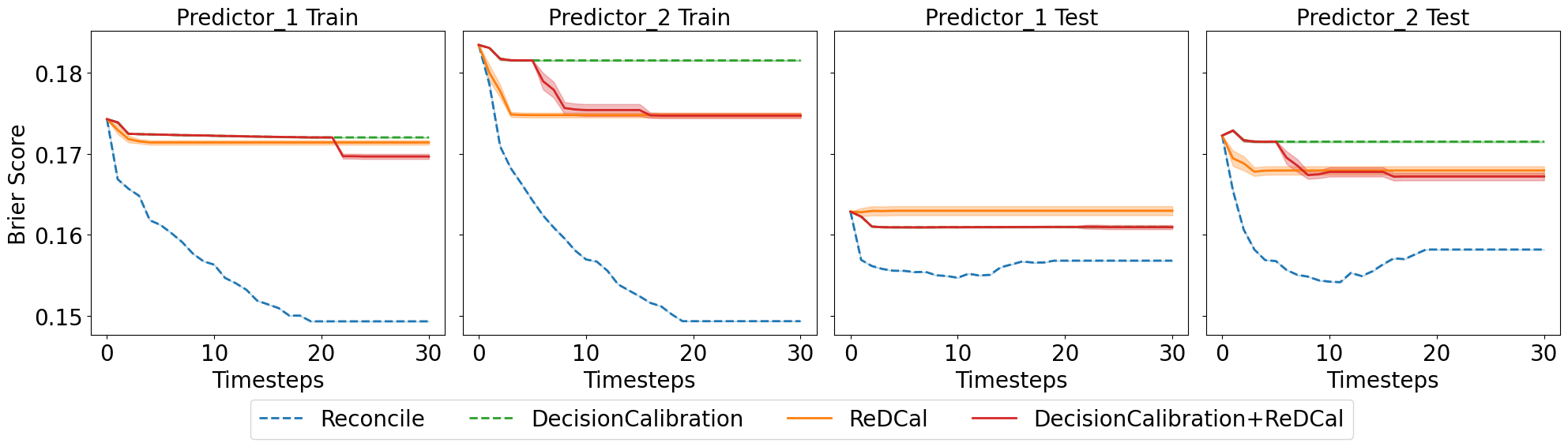}
         \caption{Brier score of the updated predictors using \Cref{alg: reconcile} (orange) and two benchmark algorithms: \Cref{alg: reconcile-brier} (dashed-blue) and \Cref{alg: decision_cali} (dashed-green). Our algorithm reduces the Brier score by a smaller amount compared to \Cref{alg: reconcile}. Results are averaged over $10$ runs and the shaded region indicates $\pm 1$ standard error.}
         \label{fig:HAM-brier-score}
\end{figure}
\paragraph{Decision loss.} In \Cref{fig:HAM_experiments}, we compare the decision loss gap of our proposed algorithm with the two baseline calibration algorithms. Compared to Reconcile, our algorithm decreases the decision loss by a larger amount on the test dataset. Furthermore, while Decision-Calibration already decreases the decision loss, our algorithm can further improve upon their result when it is used as a post-process after Decision-Calibration terminates. 
\begin{figure}[ht]
     \centering     
     \begin{subfigure}[b]{\textwidth}
         \centering
         \includegraphics[width=\textwidth]{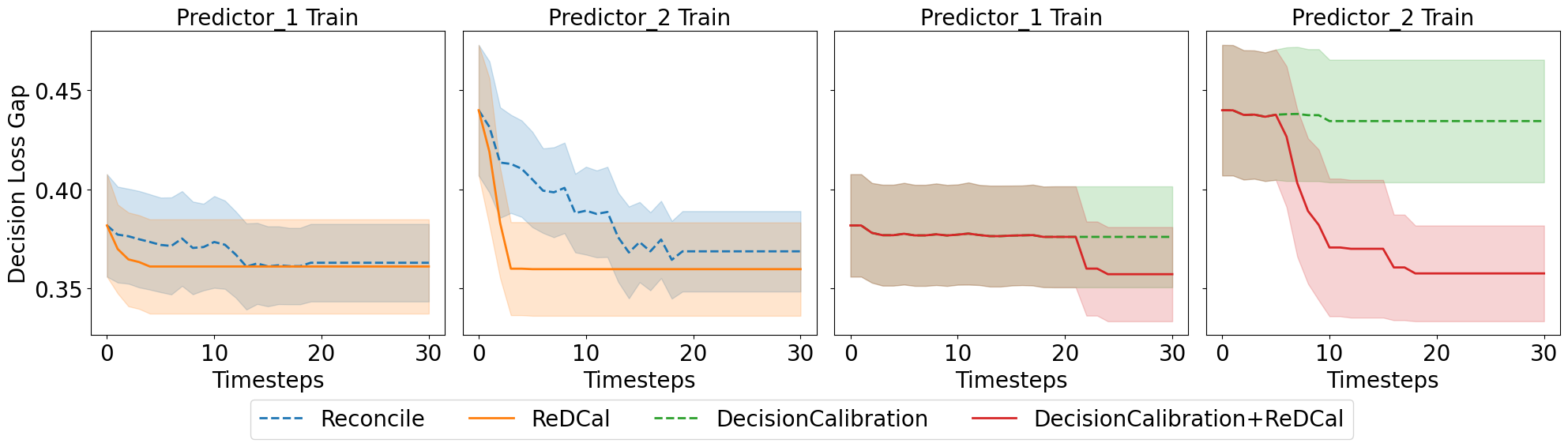}
         \caption{ReDCal decreases the decision loss on the validation partition of HAM10000 dataset. }
         \label{fig:HAM_train_loss}
     \end{subfigure}
     \vfill
     \begin{subfigure}[b]{\textwidth}
         \centering
         \includegraphics[width=\textwidth]{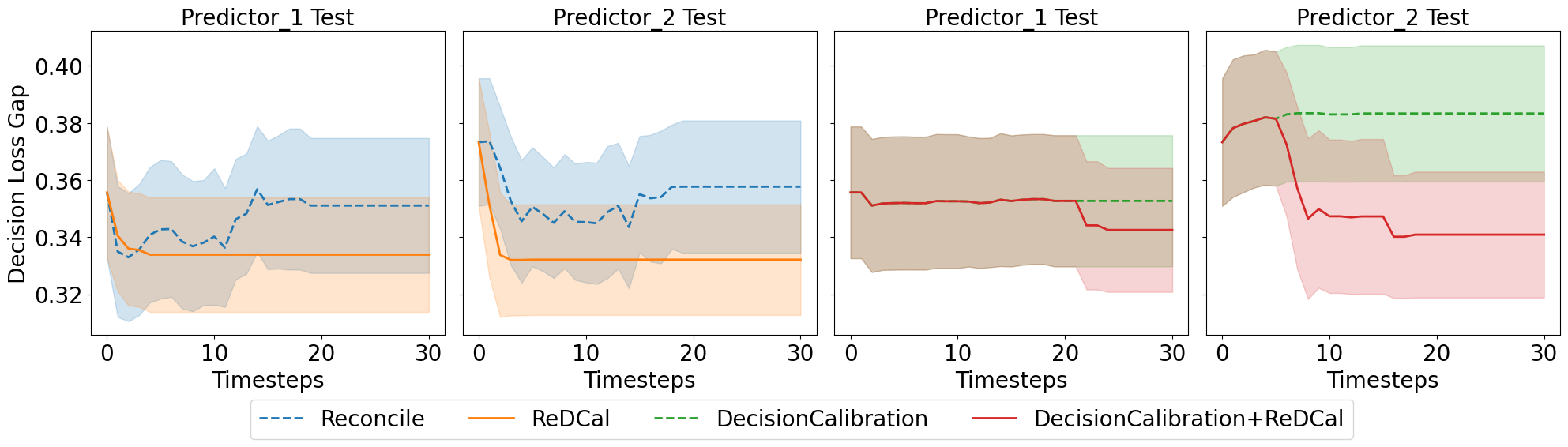}
         \caption{ReDCal decreases the decision loss on the test partition of HAM10000 dataset. }
         \label{fig:HAM_test_loss}
     \end{subfigure}
        \caption{In \Cref{fig:HAM_train_loss} and \Cref{fig:HAM_test_loss}, we plot the gap between optimal loss had we know the true label $y$ and the loss from taking best-response actions induced by the calibrated predictors on the validation set and test set, respectively. In the left two figures, we compare \Cref{alg: decision_cali} (orange) with \Cref{alg: reconcile-brier} (blue). While the average loss of predictors updated using \Cref{alg: reconcile-brier} may increase on the test set, our algorithm quickly converges and produces predictors with lower decison-making loss. In the right two figures, we compare \Cref{alg: decision_cali} (green) to \Cref{alg: decision_cali} with an additional run of \Cref{alg: reconcile} (red) as post-process. We observe that running our algorithm as post-process can still further decrease the loss compared to just running \Cref{alg: decision_cali} on its own. Results are averaged over $10$ runs and the shaded region indicates $\pm 1$ standard errors.}
        \label{fig:HAM_experiments}
\end{figure}

    

%% file: conclusion.tex
\section{Conclusion}
\label{sec:conclusion}
Predictive multiplicity is a phenomenon in machine learning where the decision-makers have two predictors with nearly equivalent accuracy but vastly different individual predictions. We propose an algorithm, ReDCal, that updates the pair of predictors using either true distribution or an i.i.d validation set until they approximately agree almost everywhere on (1) individual predictions, (2) best-response actions in the downstream decision-making task, and (3) following the best-response actions incur losses that are close to the optimal loss. This result helps alleviate the problem of predictive multiplicity in model selection. Finally, we provide experiments using real-world datasets to show that our proposed algorithm achieves lower decision loss compared to existing work. While we do not provide examples of domain-specific loss functions as part of our analysis and experiments, we hope that our findings can aid future studies on the impact of model multiplicity in decision-making.  

%% file: appendix/appendix_counterexample.tex
\section{Limitation of Prior Work (Continue)}
\label{appendix:limitation}
We provide the \Cref{alg: reconcile-brier} from \citep{roth2023reconciling} and their theoretical guarantees for completion. First, given two predictors $f_1$ and $f_2$, define the disagreement region as:
\begin{align*}
    U_\eps(f_1, f_2) \coloneqq \{ x: \abs{f_1(x) - f_2(x)} > \eps \}
\end{align*}
which can be further divided into two partitions:
\begin{align*}
    U_\eps^{>}(f_1, f_2) &= \{ x \in U_\eps(f_1, f_2) : f_1(x) > f_2(x)\} \\
    U_\eps^{<}(f_1, f_2) &= \{ x \in U_\eps(f_1, f_2): f_1(x) < f_2(x) \}
\end{align*}
\begin{algorithm}[ht]
\caption{Reconcile \citep{roth2023reconciling}}
    \begin{algorithmic}[1]
    \label{alg: reconcile-brier}
    \renewcommand{\algorithmicrequire}{\textbf{Input:}}
    \renewcommand{\algorithmicensure}{\textbf{Output:}}
    
    \REQUIRE $f_1,f_2, \eta > 0, \alpha > 0$
    \STATE Let $f_1^0 = f_1, f_2^0 = f_2$.
    \WHILE{$\mu(U_\alpha(f_1^{t_1}, f_2^{t_2})) \geq \eta$}
        \STATE For each $\bullet \in \{ > , < \}$ and $i \in \{1, 2 \}$, let:
        \begin{align*}
            v_*^\bullet = \E[y | x \in U_\eps^\bullet(f_1^{t_1}, f_2)] \quad v_i^\bullet = \E[f_i^{t_i}(x) | x \in U_\eps^\bullet(f_1^{t_1}, f_2)] 
        \end{align*}
        \STATE Let 
        \begin{align*}
            (i_t, \bullet_t) = \argmax_{i \in \{1,2\}, \bullet \in \{ >, < \}} \mu(U_\eps^\bullet (f_1^{t_1}, f_2^{t_2})) \cdot (v_*^\bullet - v_i^\bullet)^2
        \end{align*}
        breaking ties arbitrarily.
        \STATE Let:
        \begin{align*}
            g_t(x) = \begin{cases}
                1 \quad x \in U_\eps^{\bullet_t} (f_1^{t_1}, f_2^{t_2}) \\ 
                0 \quad \text{otherwise}
            \end{cases}
        \end{align*}
        \STATE Let 
        \begin{align*}
            \widetilde{\Delta}_t &= \E_{(x,y) \sim \cD}[y | g_t(x) = 1] - \E_{(x,y) \sim \cD}[f_{i_t}^{t_{i_t}} (x) | g_t(x) = 1] \\
            \Delta_t &= \mathrm{Round}(\widetilde{\Delta}_t; m)
        \end{align*}
        \STATE Let $f_i^{t_i + 1} (x) = h(x, f_i^{t_i}, g_t, \Delta_t), t_i = t_i + 1, t = t+1$.
    \ENDWHILE
    \ENSURE $(f_1^{t_1}, f_2^{t_2})$
    \end{algorithmic}
\end{algorithm}
\begin{theorem} [Reconcile \citep{roth2023reconciling}]
    For any pair of models $f_1, f_2: \cX \rightarrow [0,1]$, any distribution $\cD$, and any $\alpha, \eta > 0$, \Cref{alg: reconcile-brier} runs for $T = T_1 + T_2$ many rounds and outputs a pair of models $(f_1^{T_1}, f_2^{T_2})$ such that:
    \begin{enumerate}
        \item $T \leq (B(f_1, \cD) + B(f_2, \cD)) \cdot \frac{16}{\eta \alpha^2}$
        \item $B(f_1^{T_1}, \cD) \leq B(f_1, \cD) - T_1 \cdot \frac{\eta \alpha^2}{16}$ and $B(f_2^{T_2}, \cD) \leq B(f_2, \cD) - T_2 \cdot \frac{\eta \alpha^2}{16}$
        \item $\mu(U_\eps(f_1^{T_1}, f_2^{t_2})) \leq \eta$
    \end{enumerate}
\end{theorem}


%% file: appendix/reconcile_proof.tex
\section{Proofs of Section~\ref{sec: reconcile_distribution}: Reconcile for Decision Making} \label{appendix:reconcile_proofs}

First, we show that if a disagreement event has a large probability mass, then at least one of $f_1$, $f_2$ has a large prediction error within the region:
\begin{lemma}\label{lem: brier_gap}
    Fix any two predictors $f_1, f_2: \cX \rightarrow [0,1]^d$ and $\alpha, \eta > 0$. If $\mu(E_{\ell, a_1, a_2}) > \eta$ for some $a_1, a_2 \in \cA$, then we have 
    \begin{align}
       \|\E_{x \sim \cX}\left[f_i(x) - f^*(x)\big| E_{\ell, a_1, a_2}(x)\right]\| \geq \frac{\alpha}{2\sqrt{d}}
    \end{align}
    for some $i \in \{1,2\}$.
\end{lemma}

\begin{proof}
    By definition of event $E_{\ell, a_1, a_2}$, we have
    \begin{align*}
        \E_{x \sim \cX}\left[\ell(f_1(x), a_2) - \ell(f_1(x), a_1) + \ell(f_2(x), a_1) - \ell(f_2(x), a_2)]|E_{\ell, a_1, a_2}(x)\right] \geq \alpha
    \end{align*}
    Also, we have
    \begin{align*}
        &\E_{x \sim \cX}\left[\ell(f_1(x), a_2) - \ell(f_1(x), a_1) + \ell(f_2(x), a_1) - \ell(f_2(x), a_2)]|E_{\ell, a_1, a_2}(x)\right]\\
        =& \E_{x \sim \cX}\left[ \langle f_1(x) - f_2(x), \ell_{a_2} - \ell_{a_1}\rangle |E_{\ell, a_1, a_2}(x)\right]\\
        =& \E_{x \sim \cX}\left[ \langle f_1(x) - f^*(x), \ell_{a_2} - \ell_{a_1}\rangle |E_{\ell, a_1, a_2}(x)\right] \\
        &\quad + \E_{x \sim \cX}\left[ \langle f^*(x) - f_2(x), \ell_{a_2} - \ell_{a_1}\rangle |E_{\ell, a_1, a_2}(x)\right]\\
        \leq & \|\E_{x \sim \cX}\left[f_1(x) - f^*(x)\big| E_{\ell, a_1, a_2}(x)\right]\|_2\sqrt{d} + \|\E_{x \sim \cX}\left[f_2(x) - f^*(x) \big| E_{\ell, a_1, a_2}(x)\right]\|_2\sqrt{d}\\
        =& \sqrt{d} \cdot \left(\|\E_{x \sim \cX}\left[f_1(x) - f^*(x)\big| E_{\ell, a_1, a_2}(x)\right]\|_2 + \|\E_{x \sim \cX}\left[f_2(x) - f^*(x)\big| E_{\ell, a_1, a_2}(x)\right]\|_2\right)
    \end{align*}
    where the last inequality comes from Cauchy-Schwartz and that $\ell$ is bounded in $[0,1]$.

    Combining the above inequalities, we have 
    \begin{align*}
        \sqrt{d} \cdot \left(\norm{\E_{x \sim \cX}\left[f_1(x) - f^*(x)\big| E_{\ell, a_1, a_2}(x)\right]}_2 + \norm{\E_{x \sim \cX}\left[f_2(x) - f^*(x)\big| E_{\ell, a_1, a_2}(x)\right]}_2\right) \geq \alpha.
    \end{align*}
    Therefore, for some $i \in \{1,2\}$, we have 
    \begin{align*}
        \norm{\E_{x \sim \cX}\left[f_i(x) - f^*(x)\big| E_{\ell, a_1, a_2}(x)\right]}_2 \geq \frac{\alpha}{2\sqrt{d}}.
    \end{align*}
\end{proof}

This lemma indicates that, if we have two predictors $f_1, f_2$ that create a large disagreement event, we can falsify at least one of the models. We now show that these events also provide a directly actionable way to improve one of the models.

\begin{lemma}\label{lem: brier_after_patching}
    For any predictor $f: \cX \rightarrow [0,1]^d$, any event $E \in \mathcal{E}$, and distribution $\cD$. Let $\phi = \mathbb{E}_{(x,y) \sim \cD}[y - f(x) | E(x) = 1]$. We patch $f$ as 
    \[f'(x) = \mathrm{proj}_{[0,1]^d}(f(x) + \phi E(x)), ~~\text{where}~~ \mathrm{proj}_{[0,1]^d}(y) = \argmin_{y' \in [0,1]^d} \|y - y'\|_2.\]
    Then, 
    \[B(f, \cD) - B(f', \cD) \geq \|\phi\|^2_2 \mu(E).\]
\end{lemma}

\begin{proof}
    \begin{align*}
        B(f,\cD) - B(f', \cD) 
        =& \E \left[\|f(x) - y\|_2^2 - \|f'(x) - y\|_2^2 \right]\\
        \geq & \E \left[\|f(x) - y\|_2^2 - \|f(x) + \phi E(x) - y\|_2^2 \right] \tag{since projection is non-expansive}\\
        = & \E \left[2\langle y - f(x), \phi E(x)\rangle - \|\phi E(x)\|_2^2 \right]\\
        \geq & \|\phi\|_2^2 \cdot \mu(E)
    \end{align*}
\end{proof}

Therefore, whenever we have two predictors that have a large disagreement event, we can always falsify at least one of the predictors and improve it through patching, causing the Brier score to decrease by a large amount. Similarly, for a fixed predictor, if one of its best-response events has a large calibration error, we can patch the predictor within the event to decrease the Brier score. As the Brier score is bounded in $[0, d]$, these two observations imply that the number of time-steps for both \Cref{alg: reconcile} and its subroutine \Cref{alg: decision_cali} are bounded.

Other than the Brier score, we also care about minimizing the loss of the downstream decision-making task. We now show that, after a further update through the subroutine \Cref{alg: decision_cali}, the loss does not increase much at each time-step of \Cref{alg: reconcile}:
\begin{lemma} \label{lem: loss_each_round}
    For any predictors $f_1, f_2$, loss function $\ell \in \cL$ and any distribution $\cD$, at any time-step $t$ of \Cref{alg: reconcile}, the predictors satisfies 
    \[
    \E_{(x,y) \sim \cD} [\ell(y,\pi_\ell^\BR(f_i^{t+1}(x))) - \ell(y,\pi_\ell^\BR(f_i^{t}(x)))]
    \leq \beta \sqrt{d} K,
    \]
    for all $i \in \{1,2\}$.
\end{lemma}

\begin{proof}
    At each round, we define the set $\Delta_a^t \subseteq E^t$ as
    \begin{align*}
        \Delta_a^t = \{x \in E^t: \pi_\ell^\BR(f^{t+1}_i(x)) = a\}.
    \end{align*}
    Then, we have 
    \begin{align*}
        &\E_{(x,y) \sim \cD} [\ell(y,\pi_\ell^\BR(f_i^{t+1}(x))) - \ell(y,\pi_\ell^\BR(f_i^{t}(x)))]\\
        =& \sum_{a \in \cA}\E_{(x,y) \sim \cD} [(\ell(y, a) - \ell(y, a_{i}^t)) \Delta_a^t(x)].
    \end{align*}
    For each term in the summation, we can upper-bound it as
    \begin{align}
        &\E_{(x,y) \sim \cD} [(\ell(y, a) - \ell(y, a_i^t)) \Delta_a^t(x)]\\
        = &\langle \E_{(x,y) \sim \cD} [y \Delta_a^t(x)],  \ell_a - \ell_{a_{i}^t}\rangle \tag{Linearity of Expectation}\\
        \leq & \langle \E_{x \sim \cD_\cX} [f_i^{t+1}(x) \Delta_a^t(x)],  \ell_a - \ell_{a_{i}^t} \rangle + \beta \sqrt{d}\tag{Since $f_i^{t+1}$ is $\beta$-calibrated}\\
        \leq&  \beta \sqrt{d}. \tag{Since $a$ is the new Best-response action}
    \end{align}

    Summing these actions together, we have 
    \begin{align}
         \E_{(x,y) \sim \cD} [\ell(y,\pi_\ell^\BR(f_i^{t+1}(x))) - \ell(y,\pi_\ell^\BR(f_i^{t}(x)))]
         \leq \beta \sqrt{d} K.
    \end{align}
\end{proof}

Instead of setting a fixed $\beta$, we can calculate a different $\beta^t$ at each round, which allows a smaller increase in loss.
\begin{lemma} \label{lem: improved_beta_for_loss}
    At each round $t$, if $a_i^t$ is not the best action on $E^t$ in average, i.e. 
    \[
        \delta^t = \max_{a \in \cA} \E_{(x,y) \sim \cD}[(\ell(y, a_i^t) - \ell(y, a))E^t(x)] > 0,
    \]
    then we can set $\beta^t \leq \delta^t / \sqrt{d}$, such that 
    \[
    \E_{(x,y) \sim \cD} [\ell(y,\pi_\ell^\BR(f_i^{t+1}(x))) - \ell(y,\pi_\ell^\BR(f_i^{t}(x)))] \leq 0.
    \]
\end{lemma}
\begin{proof}
    We can write the change in loss at each round as 
    \begin{align*}
        &\E_{(x,y) \sim \cD} [\ell(y,\pi_\ell^\BR(f_i^{t+1}(x))) - \ell(y,\pi_\ell^\BR(f_i^{t}(x)))]\\
        =& \E_{(x,y) \sim \cD} [\ell(y,\pi_\ell^\BR(f_i^{t+1}(x))) - \ell(y,a_i^t))E^t(x)]\\
        =& \E_{(x,y) \sim \cD} [\ell(y,\pi_\ell^\BR(f_i^{t+1}(x))) - \ell(y,a'))E^t(x)] + \E_{(x,y) \sim \cD} [\ell(y,a') - \ell(y,a_i^t))E^t(x)]\\
        =& \sum_{a \in \cA}\E_{(x,y) \sim \cD} [(\ell(y, a) - \ell(y, a')) \Delta_a^t(x)] + \E_{(x,y) \sim \cD} [\ell(y,a') - \ell(y,a_i^t))E^t(x)],
    \end{align*}
    for any $a' \in \cA$.
    
    We can use the same analysis as in Lemma~\ref{lem: loss_each_round} to get 
    \[\sum_{a \in \cA}\E_{(x,y) \sim \cD} [(\ell(y, a) - \ell(y, a')) \Delta_a^t(x)] \leq \beta^t \sqrt{d}.\]

    For the second term, we would want the loss to be as small as possible, so we can choose $\displaystyle a' = \argmin_{a \in \cA}\E_{(x,y) \sim \cD}[\ell(y, a) \cdot E^t(x)]$ and let
    \begin{align}
        \delta^t = -\E_{(x,y) \sim \cD}[(\ell(y, a') - \ell(y, a_i^t))E^t(x)],
    \end{align}
    then $\delta^t$ is maximized and $\delta^t \geq 0$ by definition.

    The total change in loss in this round can be written as 
    \begin{align*}
        \E_{(x,y) \sim \cD} [\ell(y,\pi_\ell^\BR(f_i^{t+1}(x))) - \ell(y,\pi_\ell^\BR(f_i^{t}(x)))]
        \leq \beta^t\sqrt{d} - \delta^t.
    \end{align*}
    If $\delta^t > 0$, we can set $\beta^t \leq \delta^t / \sqrt{d}$ to ensure the loss does not increase at this round.
    
\end{proof}

\subsection{Proof of Theorem ~\ref{thm:reconcile}}
\begin{proof}
    By Lemma~\ref{lem: brier_gap} and \ref{lem: brier_after_patching}, for any $i \in \{1,2\}$, at time-step $t$, we have the inequality
    \begin{align*}
        B(f_i^t, \cD) - B(f_i^{t+1}, \cD) \geq \frac{\alpha^2\eta}{4d}.
    \end{align*}
    Taking the sum over all time-steps, we have for any $i \in \{1,2\}$,
    \begin{align*}
        B(f_i, \cD) - B(f_i^T, \cD) \geq T_i \cdot \frac{\alpha^2 \eta}{4d}. 
    \end{align*}
    Since the Brier score is always non-negative, we have 
    \begin{align*}
        T_i \leq \frac{4d \cdot B(f_i, \cD)}{\alpha^2 \eta}.
    \end{align*}

    Second, using Lemma~\ref{lem: loss_each_round} and summing over all time-steps, we have
    \begin{align}
        &\E_{(x,y) \sim \cD}[\ell(y, \pi_\ell^\BR(f_i^T(x))] - \E_{(x,y) \sim \cD}[\ell(y, \pi_\ell^\BR(f_i(x))]\\
        =& \sum_{t=1}^T \I[i_t = i] \E_{(x,y) \sim \cD} [\ell(y,\pi_\ell^\BR(f_i^{t+1}(x))) - \ell(y,\pi_\ell^\BR(f_i^{t}(x)))] \\
        \leq & T_i \cdot \beta \sqrt{d} K.
    \end{align}

    Finally, the halting condition implies that $\mu(E_{\ell, a_1, a_2}) < \eta$ for all $a_1, a_2 \in \cA$.
\end{proof}

%% file: appendix/finite_sample_proof.tex
\section{Proofs of Section~\ref{sec:finite-sample}: Finite Sample Analysis} \label{appendix:finite_sample_proofs}
First, to make our argument that in-sample quantities translate to out-sample quantities, it is useful for the patching operations to use values that are rounded to a finite grid, rather than the precise value from the arbitrary sample. We define the finite grid as follows:

\begin{definition}
    For any integer $m > 0$, let $1/m$ denote the $m+1$ grid points,
    \[\left[\frac{1}{m}\right] = \left\{0, \frac{1}{m}, \frac{2}{m}, \ldots, \frac{m-1}{m}, 1\right\}.\]
    For any value $v \in [0,1]^d$, let $Round(v; m) = \argmin_{v' \in [1/m]^d} \norm{v - v'}_2$ denote the closest grid point to $v$ in $[1/m]^d$.
\end{definition}

At each time-step in \Cref{alg: decision_cali} and \Cref{alg: reconcile}, denote $\widetilde \phi^t = Round(\phi;m)$, and we patch the predictors using $\widetilde \phi^t$ instead of $\phi^t$, \ie we update $f^t$ to $f^{t+1}$ as 
\[f^{t+1}(x) = \mathrm{proj}_{[0,1]^d}(f^t(x) + \widetilde \phi^t E^t(f^t(x), x)).\]

With this new patching operation, we can perform a similar analysis in \Cref{sec: reconcile_distribution} to show that the Brier score decreases at each iteration, and therefore the algorithm terminates within a finite number of time-steps. We denote the maximum number of time-steps, counting both \Cref{alg: decision_cali} and \Cref{alg: reconcile}, as $T_{\max}$.

Then, we can count the total number of possible predictors outputted by 
\Cref{alg: reconcile} by observing that, for a fixed pair of input predictors, each pair of output predictors can be encoded as a sequence of tuple, $\{(i^t, E^t, \Delta^t)\}_{t \in [T]}$. Here, index $i^t \in \{1,2\}$, event $E^t \in \{E_{\ell, a_1, a_2}: \ell \in \cL, a_1, a_2 \in \cA\} \cup \{E_{\ell, a} \cap E_{\ell', a_1, a_2}: \ell, \ell' \in \cL, a_1, a_2, a \in \cA\}$, and $\Delta_t \in [1/m]^d$ are all chosen from a finite set, and the length of the sequence, $T$, is also bounded. Specifically, for a fixed input $f_1, f_2$, we denote $S$ to be the set of all possible predictors outputted by \Cref{alg: reconcile}. Then, its size satisfies
\begin{align*}
    |S| \leq (4|\cL|^2K^3(m+1)^d)^{T_{\max}+1}
\end{align*}

We show that the number of predictors outputted by \Cref{alg: reconcile} is bounded:
\begin{lemma} \label{lem: num_predictor}
    Fix any pair of predictors $f_1, f_2: \cX \rightarrow [0,1]^d$ and any $\eta, \alpha, \beta > 0$. Then the total number of possible predictors outputted by \Cref{alg: reconcile} is at most $\abs{S}$ such that, for any distribution $\cD$ on which \Cref{alg: reconcile} is run, the output predictors $(f_1^{t}, f_2^{t}) \in S$. 
\end{lemma}
\begin{proof}
    First, notice that a sequence of quantities $\{(i^t, E^t, \Delta^t)\}_{t \in [T]}$ defines the pair of predictors outputted by \Cref{alg: reconcile}. 
    
    Let $S$ denote the pairs of functions induced by all such trajectories defined above. Here, $i^t \in \{1,2\}, E^t \in \{E_{\ell, a_1, a_2}: \ell \in \cL, a_1, a_2 \in \cA\} \cup \{E_{\ell, a} \cap E_{\ell', a_1, a_2}: \ell, \ell' \in \cL, a_1, a_2, a \in \cA\}$, and $\Delta_t \in [1/m]^d$. Therefore, there are
    \begin{align*}
        |S| \leq \sum_{t=1}^T \left(2(|\cL|K^2 + |\cL|^2K^3) (m+1)^d\right)^t \leq (4|\cL|^2K^3(m+1)^d)^{T_{\max}+1}
    \end{align*}
    output predictors.
\end{proof}

\subsection{Finite Grid}
With this new patching operation, we can show that the Brier score decreases on the empirical distribution $D$, corresponding to \Cref{lem: brier_after_patching}:
\begin{lemma} \label{lem: grid_brier_after_patching}
    Fix any event $E$. Let $\phi = \E_{(x,y) \sim D}[y - f(x) | E(x) = 1]$. For any predictor $f$, we patch $f$ as $f'(x) = \mathrm{proj}_\Delta(f(x) + \tilde \phi E(x))$. Then,
    \[B(f,D) - B(f', D) \geq \|\phi\|_2^2 \mu(E) - \frac{d}{4m^2}\]
\end{lemma}
\begin{proof}
    Let $\tilde f'(x) = f(x) + \phi E(x)$. Then, we have 
    \begin{align*}
        B(f,D) - B(f', D) 
        = & B(f,D) - B(\tilde f', D) + B(\tilde f', D) - B(f', D)\\
        \geq & \|\phi\|_2^2 \mu(E) + \E[\norm{f(x) + \phi E(x) - y}_2^2 - \norm{f(x) + \tilde \phi E(x) - y}_2^2] \tag{\Cref{lem: brier_after_patching}}\\
        = & \|\phi\|_2^2 \mu(E) -  \E\left[\norm{\tilde \phi - \phi}_2^2\right]\mu(E)
    \end{align*}
    By definition $\tilde \phi$, we know that each index of $|\tilde \phi - \phi|$ is in $[0,\frac{1}{2m}]$. Therefore, we have 
    \begin{align*}
        B(f,D) - B(f', D) 
        \geq & \|\phi\|_2^2 \cdot \mu(E) - \frac{d}{4m^2}.
    \end{align*}
\end{proof}

Since the Brier score is within the range $[0,d]$, and it decreases at each iteration, we can show that, if we set $m$ large enough, \Cref{alg: reconcile} terminates within a finite number of iterations:
\begin{lemma} \label{lem: sample_num_iter}
    For any predictor $f_1, f_2$. Let $m \geq \left\lceil\sqrt{\frac{d}{2\min\{\beta^2, \eta \alpha^2/4d\}}}\right\rceil$. The Brier score at each iteration of \Cref{alg: decision_cali} and $\Cref{alg: reconcile}$ satisfies
    \[B(f^{t},D) - B(f^{t+1}, D) > \frac{\min\{\beta^2, \eta \alpha^2 / (4d)\}}{2}.\]
    Counting both \Cref{alg: reconcile} and its subroutine \Cref{alg: decision_cali}, the total number of iterations $T$ satisfies
    \begin{align*}
        T \leq \frac{2d}{\min\{\beta^2, \eta \alpha^2/4d\}}.
    \end{align*}
\end{lemma}
\begin{proof}
By \Cref{alg: decision_cali}, we have by definition of $\beta$-decision calibration that
    \begin{align}
        &\mu(E^t \cap E_{\ell, a}) \cdot \norm{\E_{(x,y) \sim D}[(y - f(x) | E^t(x) \cdot E_{\ell, a}(x)=1]}_2^2 \\
        \geq& \mu(E^t \cap E_{\ell, a})^2 \cdot \norm{ \E_{(x,y) \sim D}[(y - f(x) | E^t(x) \cdot E_{\ell, a}(x)=1]}_2^2 \\
        =& \norm{\E_{(x,y) \sim D}[(y - f(x) \cdot E^t(x) \cdot E_{\ell, a}(x)]}_2^2 > \beta^2.
    \end{align}
    In \Cref{alg: reconcile}, we have by \Cref{lem: brier_gap} that 
    \begin{align}
        \mu(E^t) \norm{ \E_{(x,y) \sim D}[(y - f(x) | E^t(x)=1]}_2^2 > \frac{\eta \alpha^2}{4d}.
    \end{align}
    Therefore, for any $\phi$ and event $E$ that we patch in \Cref{alg: decision_cali} or \Cref{alg: reconcile}, they satisfy
    \[\|\phi\|_2^2 \cdot \mu(E) > \min\{\beta^2, \frac{\eta \alpha^2}{4d}\}.\]
    Letting $m \geq \left\lceil\sqrt{\frac{d}{2\min\{\beta^2, \eta \alpha^2/4d\}}}\right\rceil$, we can ensure
    \begin{align*}
        B(f^{t},D) - B(f^{t+1}, D) 
        \geq & \frac{\|\phi\|_2^2 \cdot \mu(E)}{2} > \frac{\min\{\beta^2, \eta \alpha^2 / (4d)\}}{2}.
    \end{align*}
    Since the Brier score is in the range $[0, d]$, we can bound the total number of iterations of both \Cref{alg: decision_cali} and \Cref{alg: reconcile} as
    \begin{align*}
        T \leq \frac{2d}{\min\{\beta^2, \eta \alpha^2/4d\}}.
    \end{align*}
\end{proof}

\subsection{Proof of Theorem~\ref{thm: sample_reconcile}}
First, we show that, for a fixed predictor $f$ and event $E_{\ell, a}$, the in-sample prediction error is approximately accurate. The deviation bound of the Brier score and calibration error can then be directly implied. 

\begin{lemma} \label{lem: gen_predict_error}
    Fix any $f$, $E_{\ell,a}$, with probability at least $1-\delta'$, we have
    \begin{align*}
    \norm{\E_{\cD}[(y - f(x)) E_{\ell, a}(f(x), x)] -   \frac{1}{n}\sum_{i=1}^n [(y_i - f(x_i)) E_{\ell,a}(f(x_i), x_i)]}_2  \leq \sqrt{\frac{d\ln (2d/\delta')}{2n}}.
    \end{align*}
\end{lemma}
\begin{proof}
    Fix an index $j \in [d]$, 
    we know $\E_{(x,y) \sim \cD}[(y - f(x))_j \cdot E_{\ell,a}(f(x), x)] \in [0,1]$ and 
    \[\E_D\left[\frac{1}{n}\sum_{i=1}^n [(y_i - f(x_i))_j \cdot E_{\ell,a}(f(x_i), x_i)]\right] = \E_{(x,y) \sim \cD}[(y - f(x))_j \cdot E_{\ell,a}(f(x), x)].\]

    Since $(x_i, y_i)$ is drawn i.i.d. from $\cD$, we can use Hoeffding's inequality to get, with probability $\delta' / d$,
    \begin{align*}
        &\left|\E_{(x,y) \sim \cD}[(y - f(x))_j \cdot E_{\ell,a}(f(x), x)] -\frac{1}{n}\sum_{i=1}^n [(y_i - f(x_i))_j \cdot E_{\ell,a}(f(x_i), x_i)] \right|
        \leq 
        \sqrt{\frac{\ln (2d/\delta')}{2n}}
    \end{align*}

    Using union bound, we have that with probability $1-\delta'$, the above inequality holds for all $j \in [d]$. Then, we have 
    \begin{align*}
        & \norm{\E_{(x,y) \sim \cD}[(y - f(x)) \cdot E_{\ell,a}(f(x), x)] - \frac{1}{n}\sum_{i=1}^n [(y_i - f(x_i)) \cdot E_{\ell,a}(f(x_i), x_i)]}_2 
        \\
        \leq & \sqrt{\sum_{j \in d} \left(\sqrt{\frac{\ln (2d/\delta')}{2n}}\right)^2}
        = \sqrt{\frac{d\ln (2d/\delta')}{2n}}.
    \end{align*}
\end{proof}
The deviation bound of the Brier score and calibration error can be directly implied by \Cref{lem: gen_predict_error}. We summarize them in the lemmas below:
\begin{lemma}\label{lem: gen_brier}
    For a fixed $f$, with probability at least $1-\delta'$, $|B(f, \cD) - B(f, D)| \leq \sqrt{\frac{d\ln (2d/\delta')}{2n}}$.
\end{lemma}
\begin{proof}
    Using triangle inequality, we have 
    \begin{align*}
        |B(f, \cD) - B(f, D)| = & \left|\norm{\E_{(x,y) \sim \cD}[y - f(x)]}_2 - \norm{\frac{1}{n}\sum_{i=1}^n [y_i - f(x_i)]}_2 \right|\\
        \leq& \norm{\E_{(x,y) \sim \cD}[y - f(x)] - \frac{1}{n}\sum_{i=1}^n [y_i - f(x_i)]}_2 \\
        \leq& \sqrt{\frac{d\ln (2d/\delta')}{2n}}.
    \end{align*}
\end{proof}

\begin{lemma} \label{lem: gen_cali}
    For a fixed $f$, any loss function $\ell \in \cL$ and $\cE = \{E_{\ell, a} \cap E_{\ell', a_1, a_2}: \ell, \ell' \in \cL, a, a_1, a_2 \in \cA\}$, with probability $1-\delta'$, we have
    \begin{align*}
        \norm{\E_{\cD}[(y - f(x)) E(x)] -   \frac{1}{n}\sum_{i=1}^n [(y_i - f(x_i))E(x)]}_2  \leq \sqrt{\frac{3d\ln (2dK \abs{\cL}/\delta')}{2n}}
    \end{align*}
    for all $E\in \cE$.
\end{lemma}
\begin{proof}
    The claim follows by using a union bound over the events in $\cE^t$, using \Cref{lem: gen_predict_error}, and that $|\cE^t| = K^3 \abs{\cL}^2$. 
\end{proof}

For a fixed pair of predictors, we can also show that the empirical size of the disagreement events $E_{\ell, a_1, a_2}$ is approximately correct with high probability: 
\begin{lemma}\label{lem: gen_size}
    Fix any pair of predictors $(f_1, f_2) \in S$, with probability at least $1-\delta'$ over $D$, we have
    \begin{align*}
        \left|\mu(E_{\ell,a_1,a_2}) - \frac{1}{n}\sum_{i=1}^n \I[E_{\ell,a_1,a_2}(x_i)=1]\right| \leq \sqrt{\frac{2\ln(2K \abs{\cL}/\delta')}{2n}}.
    \end{align*}
    for all $a_1, a_2 \in \cA$ with $a_1 \neq a_2$ and for all $\ell \in \cL$.
\end{lemma}
\begin{proof}
    We know $\I[E_{\ell, a_1,a_2}(x_i)=1] \in [0,1]$ and 
    \[\E_D\left[\frac{1}{n}\sum_{i=1}^n \I[E_{\ell, a_1,a_2}(x_i)\right] = \mu(E_{\ell, a_1, a_2}).\]
    Since $(x_i, y_i)$ is drawn i.i.d. from $\cD$, we can use Hoeffding's inequality to get, with probability $1-\delta'/(K^2|\cL|)$,
    \begin{align*}
        &\left|\mu(E_{\ell, a_1, a_2}) - \frac{1}{n}\sum_{i=1}^n \I[E_{\ell, a_1,a_2}(x_i)]\right|
        \leq \sqrt{\frac{2\ln(2K \abs{\cL}/\delta')}{2n}}.
    \end{align*}
    Using union bound over all pairs of $a_1, a_2 \in \cA$ and $\ell \in \cL$, we know the above inequality holds for all $a_1, a_2$ and $\ell \in \cL$ with probability at least $1-\delta'$.
\end{proof}

We summarize the above results in the theorem below. \Cref{thm: sample_reconcile} follows by solving for $n$ in the 2-4th guarantees below.
\begin{theorem} 
    Fix any distribution $\cD$ and dataset $D \sim \cD$ containing $n$ samples drawn i.i.d from $\cD$. For any pair of predictors $f_1, f_2: \cX \rightarrow [0,1]^d$, loss margin $\alpha > 0$, disagreement region mass $\eta > 0$, and decision-calibration tolerance $\beta > 0$, \Cref{alg: reconcile} run over the empirical distribution $D$ updates predictors $f_1$ and $f_2$ for $T_1$ and $T_2$ time-steps, respectively, and outputs a pair of predictors $(f_1^{T}, f_2^{T})$ such that, with probability at least $1-\delta$ over the randomness of $D \sim \cD^n$,  
    \begin{enumerate}
        \item The total number of time-steps for \Cref{alg: reconcile} and \Cref{alg: decision_cali} is 
        \[ T = T_1 + T_2 \leq \frac{2d}{\min\{\beta^2, \eta \alpha^2/4d\}} \]
        \item For $i \in \{1,2\}$, the Brier scores of the final models are lower than that of the input models:
        \begin{align*}
            B(f_i^{T_i}, \cD) \leq B(f_i, \cD) - T_i \cdot \min\left\{\beta^2, \eta \alpha^2 / (4d)\right\} + \sqrt{(d\ln (6d|S|/\delta)/(2n)}
        \end{align*}
        \item For $i \in \{1,2\}$, the downstream decision-making losses of the final models do not increase by much compared to that of the input models:
        \begin{align*}
            &\E_{(x,y) \sim \cD}[\ell(y, \pi^\BR(f_i^T(x)) - \ell(y, \pi^\BR(f_i(x))] 
            \leq \left( \beta + \sqrt{(3d\ln (6dK|S|\abs{\cL}/\delta)/(2n)} \right) \sqrt{d} K T_i
        \end{align*}
        \item The final models approximately agree on their best-response actions almost everywhere. That is, the disagreement region $E_{a_1, a_2}$ calculated using $f_1^T, f_2^T$ has small mass.
        \begin{equation*}
            \mu(E_{\ell, a_1, a_2}) \leq \eta + \sqrt{(2\ln(6K|S|\abs{\cL}/\delta)/(2n)} \quad \text{ for all }  a_1, a_2 \in \cA \quad \text{s.t } a_1 \neq a_2
        \end{equation*}
    \end{enumerate}
    Here, $S$ is the set of all possible predictors outputted by \Cref{alg: reconcile} satisfying 
    \begin{align*}
        \ln(|S|) 
        \leq \left(\frac{2d}{\min\{\beta^2, \eta \alpha^2/4d\}} + 1\right) \ln\left(4|\cL|^2K^3\left(\left\lceil\sqrt{\frac{d}{2\min\{\beta^2, \eta \alpha^2/4d\}}}\right\rceil+1\right)^d\right)
    \end{align*}
\end{theorem}

\begin{proof}
    The upper bound on $T$ holds true with probability $1$. For the remaining three guarantees, we show that each of them holds with probability at least $1-\delta/3$ over the randomness of $D$. 
    
    \textbf{Brier Score}. First, by \Cref{lem: gen_brier} and using union bound over all possible output predictors $(f_1, f_2) \in S$, we have with probability at least $1-\delta/3$ that 
    \begin{align*}
        |B(f_i, \cD) - B(f_i, D)| \leq \sqrt{\frac{d\ln (6d|S|/\delta)}{2n}}.
    \end{align*}

    By \Cref{lem: sample_num_iter}, and summing over all iterations, we have 
    \begin{align*}
        B(f_i^{T_i}, D) \leq B(f_i, D) - T_i \cdot \min\left\{\frac{\beta^2}{2}, \frac{\eta \alpha^2}{8d}.\right\}
    \end{align*}
    Therefore, 
    \begin{align*}
        B(f_i^{T_i}, \cD) 
        \leq &B(f_i^{T_i}, D) + \sqrt{\frac{d\ln (6d|S|/\delta)}{2n}} \\
        \leq &B(f_i, \cD) - T_i \cdot \min\left\{\frac{\beta^2}{2}, \frac{\eta \alpha^2}{8d}\right\} + \sqrt{\frac{d\ln (6d|S|/\delta)}{2n}}
    \end{align*}
    for $i \in \{1,2\}$.

    \textbf{Expected Loss}. 
    Using union bound over all predictors in $S$, by \Cref{lem: gen_cali}, we have, with probability at least $1-\delta/3$,
    \begin{align*}
        \norm{\E_{\cD}[(y - f(x)) E_{\ell,a}(f(x), x)] - \frac{1}{n}\sum_{i=1}^n [(y_i - f(x_i)) E_{\ell, a}(f(x_i), x_i)]}_2  
        \leq \sqrt{\frac{3d\ln (6dK|S| \abs{\cL}/\delta)}{2n}}
    \end{align*}
    for all predictors $f$, action $a \in \cA$ and loss $\ell \in \cL$.
    Using similar method as in \Cref{lem: cal_loss_est}, we define the set $\Delta_a^t \subseteq E^t$ as
    \begin{align*}
        \Delta_a^t = \{x \in E^t: \pi_\ell^{\BR}(f^{t+1}_i(x)) = a\}.
    \end{align*}
    Then, we have 
    \begin{align*}
        &\E_{(x,y) \sim \cD} [\ell(y,\pi_\ell^\BR(f_i^{t+1}(x))) - \ell(y,\pi_\ell^\BR(f_i^{t}(x)))]\\
        =& \sum_{a \in \cA}\E_{(x,y) \sim \cD} [(\ell(y, a) - \ell(y, a_{i}^t)) \Delta_a^t(x)].
    \end{align*}
    For each term in the summation, 
    \begin{align}
        &\E_{(x,y) \sim \cD} [(\ell(y, a) - \ell(y, a_i^t)) \Delta_a^t(x)]\\
        = &\langle \E_{(x,y) \sim \cD} [y \Delta_a^t(x)],  \ell_a - \ell_{a_{i}^t}\rangle \tag{Linearity of Expectation}\\
        \leq & \langle \E_{x \sim \cD_\cX} [f_i^{t+1}(x) \Delta_a^t(x)],  \ell_a - \ell_{a_{i}^t} \rangle + \left(\beta +
        \sqrt{\frac{3d\ln (6dK|S| \abs{\cL}/\delta)}{2n}}\right)\sqrt{d}\tag{\Cref{lem: gen_cali} and $\beta$-calibrated}\\
        \leq&  \left(\beta +
        \sqrt{\frac{3d\ln (6dK|S|\abs{\cL}/\delta)}{2n}}\right)\sqrt{d}. \tag{Since $a$ is the new Best-response action}
    \end{align}
    Summing these actions together, we have
    \begin{align}
         \E_{(x,y) \sim \cD} [\ell(y,\pi_\ell^\BR(f_i^{t+1}(x))) - \ell(y,\pi_\ell^\BR(f_i^{t}(x)))]
         \leq \beta\sqrt{d} K +
        \sqrt{\frac{3\ln (6dK|S| \abs{\cL}/\delta)}{2n}}d K.
    \end{align}
    Summing over all iterations, we conclude that, with probability at least $1-\delta/3$,
    \begin{align*}
        \E_{(x,y) \sim \cD} [\ell(y,\pi^\BR(f_i^{T_i}(x))) - \ell(y,\pi^\BR(f_i(x)))]
        \leq T_i (\beta \sqrt{d} K + \sqrt{\frac{3\ln (6dK|S| \abs{\cL}/\delta)}{2n}}d K)
    \end{align*}
    for all $i \in \{1,2\}$.

    \textbf{Disagreement Event}. 
    By \Cref{lem: gen_size}, with probability at least $1 - \delta/(3|S|)$, we have, for all $\ell \in \cL$, $a_1, a_2 \in \cA$ with $a_1 \neq a_2$ and all $(f_1, f_2) \in S$,
    \begin{align*}
        \left|\mu(E_{\ell, a_1,a_2}(f_1(x), f_2(x), x) - \frac{1}{n}\sum_{i=1}^n \I[E_{\ell, a_1,a_2}(f_1(x_i), f_2(x_i), x_i))]\right| \leq \sqrt{\frac{2\ln(6K|S| \abs{\cL}/\delta)}{2n}}.
    \end{align*}
    From the while loop condition in \Cref{alg: reconcile}, we know that 
    \[\frac{1}{n}\sum_{i=1}^n \I[E_{\ell, a_1, a_2}(f_1^{T_1}(x_i), f_2^{T_2}(x_i), x_i)] \leq \eta \]
    Then, using union bound over all $(f_1, f_2) \in S$ and $\ell \in \cL$, with probability at least $1-\delta/3$, we have the guarantee
    \begin{align}
        \mu(E_{\ell, a_1, a_2}(f_1^{T_1}(x), f_2^{T_2}(x), x)) 
        \leq  
        &\frac{1}{n}\sum_{i=1}^n \I[E_{\ell, a_1, a_2}(f_1^{T_1}(x_i), f_2^{T_2}(x_i), x_i)] + \sqrt{\frac{2\ln(6K|S| \abs{\cL}/\delta)}{2n}} \\
        \leq
        &\eta + \sqrt{\frac{2\ln(6K|S|\abs{\cL}/\delta)}{2n}}
    \end{align}
    for all $a_1, a_2 \in \cA, a_1 \neq a_2$ and $\ell \in \cL$. 

    Finally, using results from \Cref{lem: num_predictor}, value of $T_{\max}$, and $m = \left\lceil\sqrt{\frac{d}{2\min\{\beta^2, \eta \alpha^2/4d\}}}\right\rceil$, we conclude by showing
    \begin{align*}
        \ln(|S|) 
        \leq & \ln \left((4|\cL|^2K^3(m+1)^d)^{T_{\max}+1}\right)\\
        = &(T_{\max}+1) \ln(4|\cL|^2K^3(m+1)^d)\\
        = & \left(\frac{2d}{\min\{\beta^2, \eta \alpha^2/4d\}} + 1\right) \ln\left(4|\cL|^2K^3\left(\left\lceil\sqrt{\frac{d}{2\min\{\beta^2, \eta \alpha^2/4d\}}}\right\rceil+1\right)^d\right)
    \end{align*}
\end{proof}